\documentclass{article}
\usepackage{iclr2021_conference,times}
\usepackage{hyperref}
\usepackage{url}

\usepackage[pdftex]{graphicx}
\usepackage{amsmath,amsfonts,bm}
\usepackage{pifont,bbm,nicefrac}
\usepackage{subcaption,booktabs}
\usepackage{amsthm}
\newcommand{\ie}{\textit{i.e.}~}
\newcommand{\eg}{\textit{e.g.}~}
\newtheorem{lemma}{Lemma}
\newtheorem{theorem}{Theorem}
\newtheorem{assumption}{Assumption}
\newtheorem{proposition}{Proposition}
\newcommand{\circledOne}{\text{\ding{172}}\;}
\newcommand{\circledTwo}{\text{\ding{173}}\;}
\newcommand{\circledThree}{\text{\ding{174}}\;}

\title{Expected Tight Bounds for Robust Training}

\author{
Salman Alsubaihi\textsuperscript{*\S}
\& Adel Bibi\textsuperscript{*\ddag\S}
\& Modar Alfadly\textsuperscript{*\S}
\& Abdullah Hamdi\textsuperscript{*}
\& Bernard Ghanem\textsuperscript{*}\\
\textsuperscript{*} KAUST \& \textsuperscript{\ddag} University of Oxford\\
\textsuperscript{\S} Equal contribution
}

\iclrfinalcopy 
\begin{document}

\maketitle

\begin{abstract}
    Training deep neural networks that are robust to norm-bounded adversarial attacks remains an elusive problem. While exact and inexact verification-based methods are generally too expensive to train large networks, it was demonstrated that bounded input intervals can be inexpensively propagated from a layer to another through deep networks. This interval bound propagation approach (IBP) not only has improved both robustness and certified accuracy but was the first to be employed on large/deep networks. However, due to the very loose nature of the IBP bounds, the required training procedure is complex and involved. In this paper, we closely examine the bounds of a block of layers composed in the form of Affine-ReLU-Affine. To this end, we propose \emph{expected} tight bounds (true bounds in expectation), referred to as ETB, which are provably tighter than IBP bounds in expectation. We then extend this result to deeper networks through blockwise propagation and show that we can achieve orders of magnitudes tighter bounds compared to IBP. Furthermore, using a simple standard training procedure, we can achieve impressive robustness-accuracy trade-off on both MNIST and CIFAR10.
\end{abstract}

\section{Introduction}
Deep neural networks (DNNs) are susceptible to small imperceptible perturbations, best known as \emph{adversarial attacks} which can lead to drastic performance degradation. Several network defense approaches were proposed to alleviate this problem. They can be coarsely categorized into empirical methods, like adversarial training \citep{madry2017towards}, and provably verifiable methods \citep{katz2017reluplex}. Such verifiers are generally very expensive to compute exactly due to their combinatoric nature. However, they can be sped up with relaxed verification, often referred to as certification methods, by over approximating the worst adversarial loss, over all bounded energy (commonly measured in $\ell_\infty$) perturbations around a given input \citep{kolter2017provable}. It has been demonstrated \citep{gowal2018effectiveness} that robustly training large networks using a certificate is possible by leveraging the cheap-to-compute but very loose interval-based certifier, known as interval domain from \citet{mirman2018differentiable}. In particular, they propagate the $\epsilon$-$\ell_\infty$ norm bounded input centered at $\mathbf{x} \in \mathbb{R}^n$ (\ie $[\mathbf{x}-\epsilon\mathbf{1}_n,\mathbf{x} + \epsilon\mathbf{1}_n]$) through the layers in the network. Despite the simplicity and cost-efficacy of this interval bound propagation (IBP), it results in very loose output interval bounds; which in turn necessitates a complex and carefully tuned training procedure. We aim in this work to improve the bounds computed by IBP to allow for easier training, which can result in stronger robustness, by investigating the bounds in a probabilistic expectation setting. Figure \ref{fig:pull_figure} demonstrates the main differences between IBP and our proposed method ETB. We were inspired by prior arts that had investigated the notion of probabilistic certification \citep{webb2018statistical,weng2018evaluating}.

\begin{figure}[t]
    \begin{center}
        \includegraphics[width=0.99\linewidth]{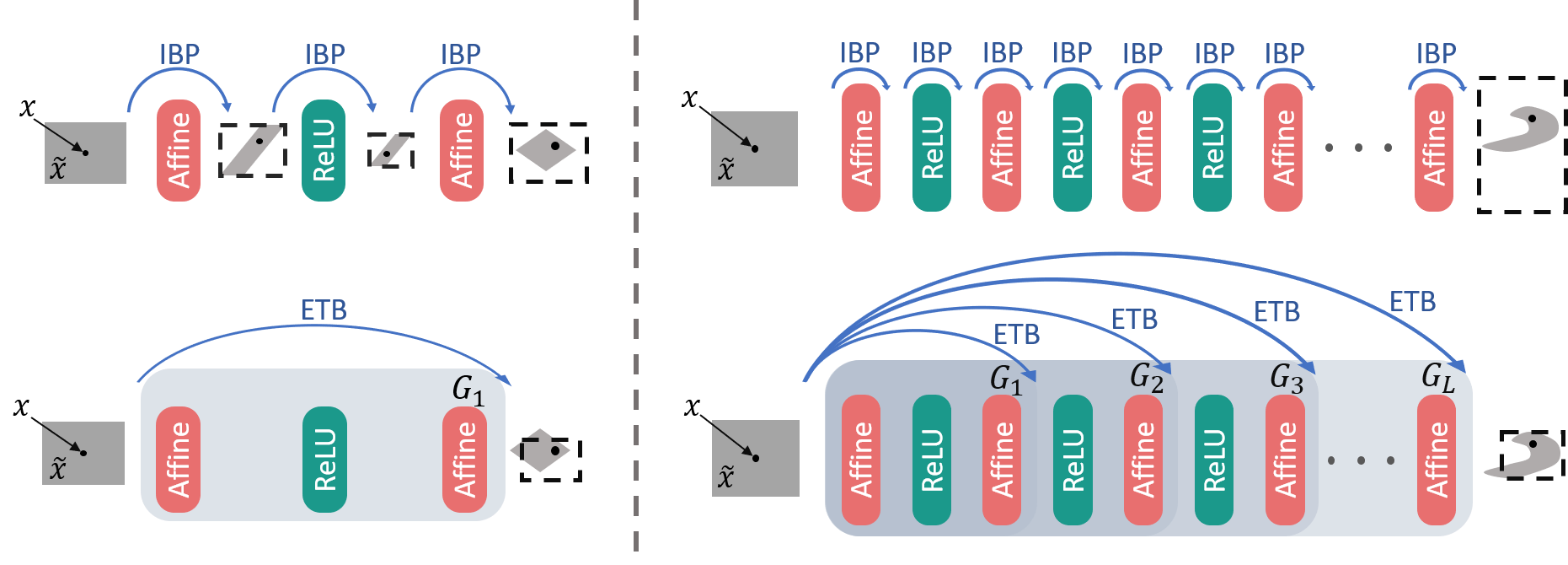}
    \end{center}
    \caption{\textbf{ETB vs IBP bound propagation.} ETB propagates input interval bounds through a block of layers by approximating the block with an intermediate linear layer $\mathbf{G}_i$ as opposed to propagating them layerwise in IBP. On the right, we show how to extend ETB for deep networks by the recursive blockwise propagation. Observe that the interval bounds computed by ETB are not always supersets to the output polytope but they can be much tighter than IBP.}
    \label{fig:pull_figure}
\end{figure}

We study the output bounds of the block in the form of Affine-ReLU-Affine in expectation over a distribution of network parameters. We prove that our bounds (ETB) are, in expectation, supersets to the true bounds of this block and much tighter than IBP bounds \citep{gowal2018effectiveness}. We conduct experiments on synthetic and on real networks, to verify the theory, as well as the factors of improvement over IBP. Additionally, we show a practically-efficient approach to propagating our blockwise bounds through a composition of blocks constituting a deep network; thus, resulting in magnitudes tighter output bounds compared to IBP. Due to the tightness of our proposed expected bounds, we show that with a simple standard training procedure, large deep networks can be robustly trained on both MNIST \citep{lecun1998mnist} and CIFAR10 \citep{krizhevsky2009learning} achieving state-of-art robustness-accuracy trade-off compared to IBP. In other words, we can consistently improve robustness with minimal effect on test accuracy as compared to IBP.

\section{Methodology}

\subsection{Interval Bound Propagation (IBP)}
Given input interval bounds for a network (\ie $[\mathbf{L_x, U_x}]$), \citet{gowal2018effectiveness} proposed obtaining the output bounds (\ie $[\mathbf{L_{IBP}}, \mathbf{U_{IBP}}]$) by recursively computing them for every layer.

The output bounds for affine layers in the form $f(\mathbf{x}) = \mathbf{Ax+b}$, where $\mathbf{A} \in \mathbb{R}^{k \times n}$ and $\mathbf{b} \in \mathbb{R}^k$:
\begin{equation}
    \label{eq:bound_through_linear}
    \mathbf{L}_f, \mathbf{U}_f
    = f\left(\frac{\mathbf{U_x} + \mathbf{L_x}}{2}\right) \mp |\mathbf{A}|\left(\frac{\mathbf{U_x} - \mathbf{L_x}}{2}\right)
\end{equation}
Where $|.|$ is an elementwise absolute value operator. For the sake of brevity, the operator $\mp$ is a subtraction for $\mathbf{L}_f$ and an addition for $\mathbf{U}_f$. When the input is $\epsilon$-$\ell_\infty$ norm bounded (\ie $\mathbf{L_x, U_x} = \mathbf{x} \mp \epsilon\mathbf{1}_n$ where $\mathbf{1}_n$ is the $\mathbb{R}^n$ all-ones vector), Equation (\ref{eq:bound_through_linear}) simply becomes $\mathbf{L}_f, \mathbf{U}_f = f(\mathbf{x}) \mp \epsilon|\mathbf{A}|\mathbf{1}_n$.

For any non-decreasing elementwise nonlinearity $h(\mathbf{x})$, like the ReLU \emph{activation function} (\ie $h(\mathbf{x}) = \max(\mathbf{x}, \mathbf{0}_n)$), the bounds are propagated as $\mathbf{L}_h, \mathbf{U}_h = h(\mathbf{L_x}), h(\mathbf{U_x})$.

\subsection{Expected Tight Bound Propagation (ETB)} \label{sec:etb}
The high level idea is to use a recursive blockwise bound propagation which can be tighter than applying IBP layerwise for deeper networks. We consider the Affine-ReLU-Affine block of the functional form $g(\mathbf{x}) = \mathbf{a}_2^\top \max\left(\mathbf{A}_1\mathbf{x} + \mathbf{b}_1, \mathbf{0}_n\right) + b_2$. Throughout the paper, we take $\mathbf{A}_1 \in \mathbb{R}^{k \times n}$ and without loss of generality the second affine map is a single vector $\mathbf{a}_2 \in \mathbb{R}^{k}$. We propose new closed form expressions for the interval bounds denoted as $[\mathbf{L}_{\mathbf{ETB}},\mathbf{U}_{\mathbf{ETB}}]$. In expectation under a distribution of the network parameters $\theta = \{\mathbf{A}_1, \mathbf{a}_2\}$, we prove them to be supersets to the true (tightest) bounds $[\mathbf{L}_{g},\mathbf{U}_{g}]$ for a sufficiently large input dimension $n$ and tighter than $[\mathbf{L}_{\mathbf{IBP}},\mathbf{U}_{\mathbf{IBP}}]$ as the number of hidden nodes $k$ increases. To derive our ETB bounds, we study the bounds of:
\begin{equation}
    \tilde{g}(\tilde{\mathbf{x}}) = \mathbf{a}_2^\top\mathbf{M}\left(\mathbf{A}_1\tilde{\mathbf{x}} + \mathbf{b}_1\right) + b_2.
\end{equation}
Note that $\tilde{g}$ is very similar to $g$ but with the ReLU replaced by a diagonal matrix $\mathbf{M} = \text{diag}\left(\mathbbm{1}\left\{\mathbf{U}_1 \ge \mathbf{0}_k\right\}\right)$, where $\mathbbm{1}$ is the indicator function and $\mathbf{U}_1$ is the upper bound of the first affine layer (\ie $\mathbf{A}_1\tilde{\mathbf{x}} + \mathbf{b}_1$). In other words, $\mathbf{M}_{ii}$ is one when the $i^{\text{th}}$ element of $\mathbf{U}_1$ is non-negative and zero otherwise. Observe that for the coordinates with $\mathbf{U}_1 \leq 0$, \ie the output of the first affine layer is always negative, the function $\tilde{g}$ behaves identically to $g$. Note that for a given $\mathbf{M}$, $\tilde{g}$ is an affine function with the following output interval bounds:
\begin{equation}
    \label{eq:etb}
    \mathbf{L}_{\mathbf{ETB}}, \mathbf{U}_{\mathbf{ETB}} = \tilde{g}\left(\frac{\mathbf{U_x} + \mathbf{L_x}}{2}\right) \mp |\mathbf{a}_2^\top \mathbf{M} \mathbf{A}_1|\left(\frac{\mathbf{U_x} - \mathbf{L_x}}{2}\right)
\end{equation}
To compare $\mathbf{L}_{\mathbf{ETB}}$ and $\mathbf{U}_{\mathbf{ETB}}$ to $\mathbf{L}_{g}$ and $\mathbf{U}_{g}$, respectively, and since having access to $\mathbf{L}_{g}$ and $\mathbf{U}_{g}$ is not feasible, we state the following trivial assumption.

\begin{assumption}
    \label{key_assumption}
    Let $\mathbf{A}_1 \sim \mathcal{N}(\mathbf{0}, \sigma_{\mathbf{A}_1}^2\mathbf{I})$, $\mathbf{a}_2 \sim \mathcal{N}(\mathbf{0}_n, \sigma_{\mathbf{a}_2}^2\mathbf{I})$, and $\tilde{\mathbf{x}} \sim \mathcal{U}[\mathbf{x} - \epsilon \mathbf{1}_n,\mathbf{x} + \epsilon\mathbf{1}_n]$ where $\mathbb{E}_{\mathbf{A}_1,\mathbf{a}_2,\tilde{\mathbf{x}}} \left[g(\tilde{\mathbf{x}})\right]$ and $\text{Var}_{\mathbf{A}_1,\mathbf{a}_2,\tilde{\mathbf{x}}} \left[g(\tilde{\mathbf{x}}) \right]$ are finite. Then, for a sufficiently large $m$:
    \begin{align*}
        & \mathbf{L}_{\text{approx}} \leq \mathbb{E}_{\mathbf{A}_1,\mathbf{a}_2}\left[\mathbf{L}_{g}\right]
        \text{~and~}
        \mathbb{E}_{\mathbf{A}_1,\mathbf{a}_2}\left[\mathbf{U}_{g}\right] \leq \mathbf{U}_{\text{approx}},
        \text{~where} \\
        & \mathbf{L}_{\text{approx}}, \mathbf{U}_{\text{approx}} = \mathbb{E}_{\mathbf{A}_1,\mathbf{a}_2,\tilde{\mathbf{x}}} \left[g(\tilde{\mathbf{x}})\right] \mp m \sqrt{\text{Var}_{\mathbf{A}_1,\mathbf{a}_2,\tilde{\mathbf{x}}} \left[g(\tilde{\mathbf{x}}) \right]}
    \end{align*}
\end{assumption}

Under Assumption \ref{key_assumption} and to show that $[\mathbf{L}_{\mathbf{ETB}}, \mathbf{U}_{\mathbf{ETB}}]$ form a superset to $[\mathbf{L}_{g},\mathbf{U}_{g}]$ in expectation, we can show that $\mathbb{E}_{\mathbf{A}_1,\mathbf{a}_2}\left[\mathbf{L}_{\mathbf{ETB}}\right] \leq \mathbf{L}_{\text{approx}}$ and $\mathbf{U}_{\text{approx}} \leq \mathbb{E}_{\mathbf{A}_1,\mathbf{a}_2}\left[\mathbf{U}_{\mathbf{ETB}}\right]$. In general, $\mathbf{L}_{\text{approx}}$ and $\mathbf{U}_{\text{approx}}$ are difficult to compute but they can be well approximated as follows:

\begin{proposition}
    \label{lyapunov_clt}
    For independent $\mathbf{a} \in \mathbb{R}^n \sim \mathcal{N}(\mathbf{0}_n,\sigma_a^2 \mathbf{I})$ and $\tilde{\mathbf{x}} \sim \mathcal{U}[\mathbf{x} - \epsilon \mathbf{1}_n,\mathbf{x} + \epsilon\mathbf{1}_n]$, we have:
    \begin{align*}
        & \frac{\sum_{i=1}^n (\tilde{x}_i a_i - \mathbb{E}[a_i \tilde{x}_i])}
        {\sqrt{\text{Var}\left(\sum_{i=1}^n (\tilde{x}_i a_i - \mathbb{E}[a_i \tilde{x}_i]) \right)}}
        \rightarrow^d \mathcal{N}(0,1)
    \end{align*}
    where $\rightarrow^d$ indicates convergence in distribution (Lyapunov Central Limit Theorem).
\end{proposition}

\begin{proposition}
    \label{output_cov_clt}
    For independent $\mathbf{A}_1 \in \mathbb{R}^{k \times n} \sim \mathcal{N}(\mathbf{0}, \sigma_{\mathbf{A}_1}^2\mathbf{I})$ and $\tilde{\mathbf{x}} \sim \mathcal{U} \left[\mathbf{x} - \epsilon \mathbf{1}_n, \mathbf{x} + \epsilon\mathbf{1}_n\right]$,
    \begin{align*}
        \text{Covariance}(\mathbf{A}_1\tilde{\mathbf{x}}) = \left(\frac{1}{3}\epsilon^2 \sigma_{\mathbf{A}_1}^2 n + \sigma_{\mathbf{A}_1}^2 \text{trace}(\mathbf{x} \mathbf{x}^\top)\right) \mathbf{I}.
    \end{align*}
\end{proposition}

For a sufficiently large $n$, following Propositions \ref{lyapunov_clt} and \ref{output_cov_clt}, we can approximate $g(\tilde{\mathbf{x}}) \approx \mathbf{a}_2^\top \max\left(\tilde{\mathbf{y}},\mathbf{0}_n\right) + b_2$, where $\tilde{\mathbf{y}}$ by Lyapunov Central Limit Theorem is a Gaussian vector $\tilde{\mathbf{y}} \sim \mathcal{N}(\mathbf{b}_1, (\frac{1}{3}\epsilon^2 \sigma_{\mathbf{A}_1}^2 n + \sigma_{\mathbf{A}_1}^2 \text{trace}(\mathbf{x} \mathbf{x}^\top)) \mathbf{I})$. This now can be used to approximate $\mathbf{L}_{\text{approx}}$ and $\mathbf{U}_{\text{approx}}$.

\begin{theorem}
    \label{thm:correctness_thm1}
    (ETB as Supersets in Expectation) Under Assumption \ref{key_assumption} and a large dimension $n$,
    \begin{equation}
        \mathbb{E}_{\theta}\left[\mathbf{L}_{\mathbf{ETB}}\right] \leq \mathbb{E}_{\theta}\left[\mathbf{L}_{\text{g}}\right] \text{~and~} \mathbb{E}_{\theta}\left[\mathbf{U}_{\text{g}}\right] \leq \mathbb{E}_{\theta}\left[\mathbf{U}_{\mathbf{ETB}}\right].
    \end{equation}
\end{theorem}

Theorem \ref{thm:correctness_thm1} states that the interval bounds for function $\tilde{g}$ are simply looser bounds to the function of interest $g$ in expectation under plausible distributions of $\theta = \{\mathbf{A}_1, \mathbf{a}_2\}$.

Now, we investigate the tightness of these bounds as compared to the IBP bounds $[\mathbf{L}_{\mathbf{IBP}}, \mathbf{U}_{\mathbf{IBP}}]$.
\begin{theorem}
    \label{thm:tightness_thm2}
    (ETB vs. IBP in Expectation) Consider an input $\tilde{\mathbf{x}} \sim \mathcal{U}[\mathbf{x} - \epsilon \mathbf{1}_n,\mathbf{x} + \epsilon\mathbf{1}_n]$ to $\mathbf{a}_2^\top \max\left(\mathbf{A}_1\mathbf{x} + \mathbf{b}_1, \mathbf{0}_n\right) + b_2$, where $\mathbf{a}_2\sim\mathcal{N}(\mathbf{0}_n,\sigma_{\mathbf{a}_2}^2\mathbf{I})$, under the assumption that $\forall j$
    \begin{equation*}
        \frac{1}{\sqrt{2\pi}} \mathbf{x}_j \mathbf{1}_k^\top \mathbf{A}_1(:,j) + \frac{1}{2n} \mathbf{1}_k^\top \mathbf{b}_1
        \ge \epsilon\left( \|\mathbf{A}_1(:,j)\|_2 - \frac{1}{\sqrt{2\pi}} \|\mathbf{A}_1(:,j)\|_1\right)
    \end{equation*}
    we have: $\mathbb{E}_{\mathbf{a}_2} \left[(\mathbf{U}_{\textbf{IBP}} - \mathbf{L}_{\textbf{IBP}}) - (\mathbf{U}_\mathbf{ETB} - \mathbf{L}_\mathbf{ETB})\right] \ge 0$.
\end{theorem}

Theorem \ref{thm:tightness_thm2} states that under some assumptions on $\mathbf{A}_1$ and a plausible distribution for $\mathbf{a}_2$, our proposed ETP interval width can be much smaller than the IBP interval width in expectation. Next, we show that the inequality assumption in Theorem \ref{thm:tightness_thm2} is very mild. In fact, a wide range of $(\mathbf{A}_1,\mathbf{b}_1)$ satisfy it, and the following proposition gives an example.

\begin{proposition}
    \label{proposition_random_mat}
    For a matrix $\mathbf{A}_1 \in \mathbb{R}^{k \times n} \sim \mathcal{N}(\mathbf{0}, \mathbf{I})$,
    \begin{equation*}
         \mathbb{E}_{\mathbf{A}_1}\left(\|\mathbf{A}_1(:,j)\|_2 - \frac{1}{\sqrt{2\pi}} \|\mathbf{A}_1(:,j)\|_1\right)
         = \sqrt{2} \frac{\Gamma\left(\frac{k+1}{2}\right)}{\Gamma\left(\frac{k}{2}\right)} - k \sqrt{\frac{2}{\pi}}
         \approx \sqrt{k} - k\sqrt{\frac{2}{\pi}}.
    \end{equation*}
\end{proposition}

Proposition \ref{proposition_random_mat} implies that as the number of hidden nodes $k$ increases, the expectation of the right hand side of the inequality assumption in Theorem \ref{thm:tightness_thm2} grows more negative, while the left hand side is zero in expectation when $\mathbf{b}_1\sim\mathcal{N}(\mathbf{0}_k,\mathbf{I})$. In other words, for zero-mean Gaussian weights $(\mathbf{A}_1,\mathbf{b}_1)$ and with a large enough number of hidden nodes $k$, the assumption is satisfied. Furthermore, it is common to regularize network weights while training with an $\ell_2$ regularizer; encouraging the weights to follow a zero-mean Gaussian distribution. We show empirical evidence of this on MNIST and CIFAR10 in the \textbf{Appendix} along with all the proofs and detailed analyses.

\section{Experiments} \label{sec:training}

\begin{figure*}[t]
    \centering
    \includegraphics[width=\linewidth]{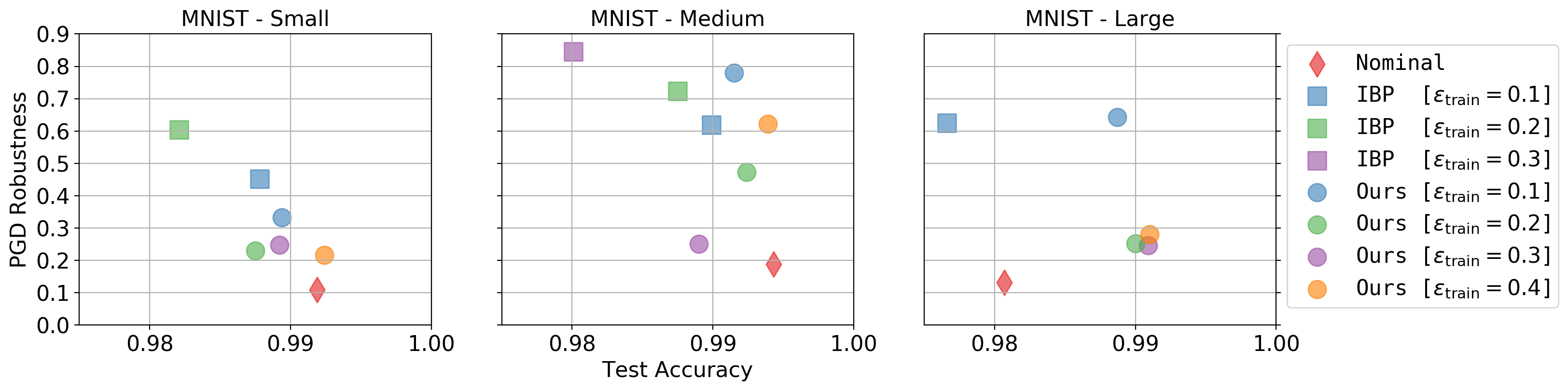}
    \includegraphics[width=\linewidth]{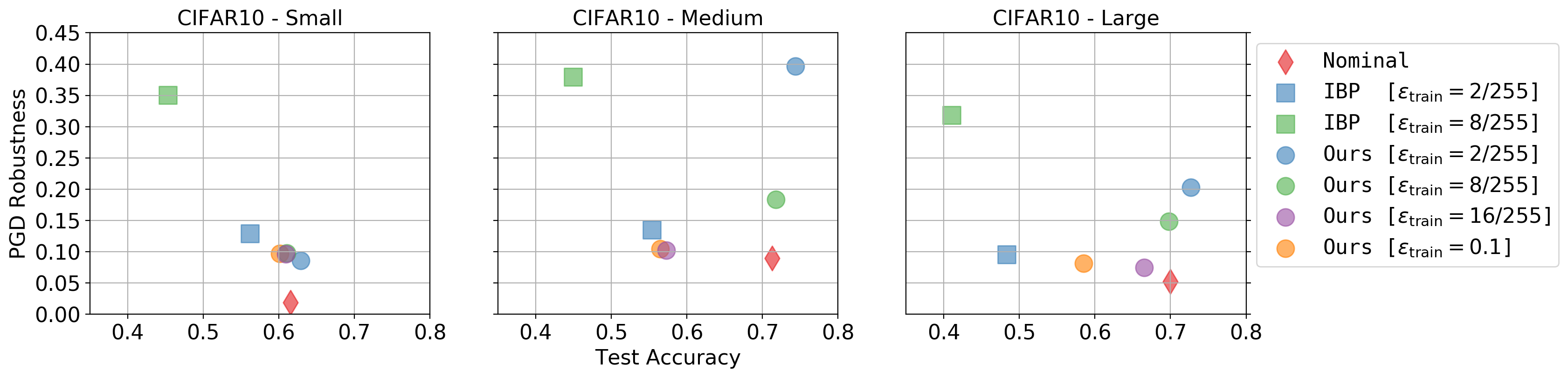}
    \caption{\footnotesize\textbf{Better Test Accuracy and Robustness on MNIST and CIFAR10.} We compare test accuracy and PGD robustness averaged over multiple $\epsilon_\text{test}$ of three models (small, medium, and large) robustly trained using our bounds against IBP. We have trained both methods using four different $\epsilon_\text{train}$. We eliminated all models with test accuracy lower than $97.5$\% for MNIST and $40.0$\% for CIFAR10. This shows an impressive trade-off between accuracy and robustness where in some cases we even excel on medium and large models.}\label{fig:mnist_cifar10}
\end{figure*}

We compare our method against models trained nominally (\ie only the standard training loss is used), and those trained robustly with IBP \citep{gowal2018effectiveness}. Given the well-known robustness-accuracy trade off \citep{tsipras2018robustness}, robust models are often less accurate. Therefore, we compare all methods using a robustness vs. accuracy scatter plot. Following prior work, we use Projected Gradient Descent (PGD) \citep{madry2017towards} to measure robustness. We use a loss function similar to the one proposed in \citet{gowal2018effectiveness}. In particular, we use $L = \ell(f_{\theta}(\mathbf{x}),\mathbf{y}_{g}) + \kappa \ell(\mathbf{z},\mathbf{y}_{g})$, where $\ell$, $f_{\theta}(\mathbf{x})$, $\mathbf{y}_{g}$, and $\kappa$ are the cross-entropy loss, output logits, true class label, and regularization hyperparameter, respectively. $\mathbf{z}$ represents the ``adversarial'' logits that combine the lower bound of the true label and the upper bound of all other labels. Nominal training occurs when $\kappa=0$. Due to the tightness of our bounds, in contrast to IBP, we don't need to carefully tune $\kappa$ or $\epsilon_\text{train}$.

We train the three network models (small, medium, large) provided by \citet{gowal2018effectiveness} on both MNIST and CIFAR10. Following the same setup in \citet{gowal2018effectiveness}, we train all models with $\epsilon_\text{train} \in \{0.1,0.2,0.3,0.4\}$ for MNIST and $\epsilon_\text{train} \in \{2/255, 8/255, 16/255, 0.1\}$ for CIFAR10. In all experiments, and for stronger baselines and fair comparison between IBP training and ETB training, we grid search over $\{0.1,0.001,0.0001\}$ learning rates and employ a temperature over the logits with a grid of $\{1,1/5\}$ as in \citet{hinton2015distilling} and report the best performing models for both. Then, we compute PGD robustness for every $\epsilon_\text{train}$ of every model for all $\epsilon_\text{test} \in \{0.1,0.2,0.3,0.4\}$ for MNIST and for all $\epsilon_\text{test} \in \{2/255, 8/255, 16/255, 0.1\}$ for CIFAR10.

We compute the average PGD robustness over all $\epsilon_\text{test}$ and the test accuracy, and report them in a 2D scatter plot. We report the performance results on MNIST and CIFAR10 for the small, medium, and large architectures in Figure \ref{fig:mnist_cifar10}. For all trained architectures, we only report the results for those that achieve at least a test accuracy of $97.5\%$ and $40\%$ on MNIST and CIFAR10, respectively; otherwise, it is an indication of failure in training. Interestingly, our training scheme can be used to train all architectures for all $\epsilon_\text{train}$. This is unlike IBP, which for example was only able to successfully train the large architecture with $\epsilon_\text{train} = 0.1$ on MNIST. Moreover, models trained with ETB always achieve better PGD robustness on all architectures while preserving similar if not higher accuracy (on large networks). Models trained with IBP achieve high robustness but their test accuracy is drastically affected. More details and experiments are left for the avid reader in the \textbf{Appendix}.

\clearpage
\bibliography{references}
\bibliographystyle{iclr2021_conference}
\clearpage

\appendix
\setcounter{lemma}{0}
\setcounter{theorem}{0}
\setcounter{assumption}{0}
\setcounter{proposition}{0}
\section{PGD Robustness on Specific Input Bounds} \label{sec:pgd_robustness}
In section \ref{sec:training}, we used the three models (small, medium, and large) provided by \citet{gowal2018effectiveness}:

\begin{table}[h]
    \centering
    \begin{tabular}{c|c|c}
        \toprule
        small                           & medium                           & large                             \\
        \midrule
        CONV $16 \times 4 \times 4 + 2$ & CONV $32 \times 3 \times 3 + 1$  & CONV $64 \times 3 \times3 + 1$    \\
        CONV $32 \times 4 \times 4 + 1$ & CONV $32 \times 4 \times 4 + 2$  & CONV $64 \times 3 \times 3 + 1$   \\
        FC 100                          & CONV $64  \times 3 \times 3 + 1$ & CONV $128 \times 3 \times 3  + 2$ \\
                                        & CONV $64 \times 4 \times 4 + 2$  & CONV $128 \times 3 \times 3 + 1$  \\
                                        & FC 512                           & CONV $128 \times 3 \times 3 + 1$  \\
                                        & FC 512                           & FC 200                            \\
        \bottomrule
    \end{tabular}
    \caption{\textbf{Model Architectures.} ``CONV $p \times w \times h + s$'', correspond to $p$ 2D convolutional filters with size $(w \times h)$ and strides of $s$. While ``FC $d$'' is a fully connected layer with $d$ outputs.}\label{tab:architecture}
\end{table}

\begin{figure*}[h]
    \centering
    \includegraphics[width=\linewidth]{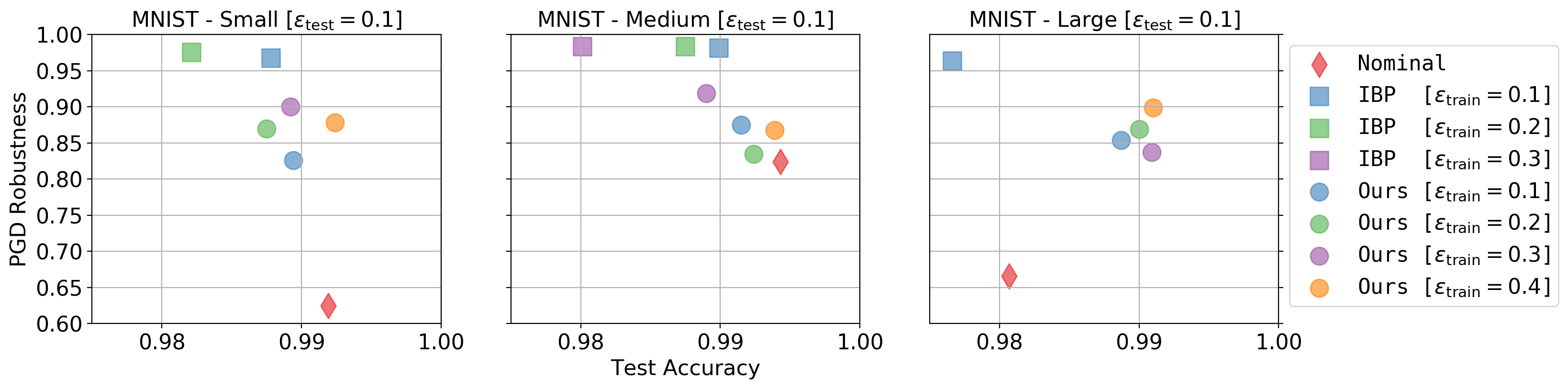}
    \includegraphics[width=\linewidth]{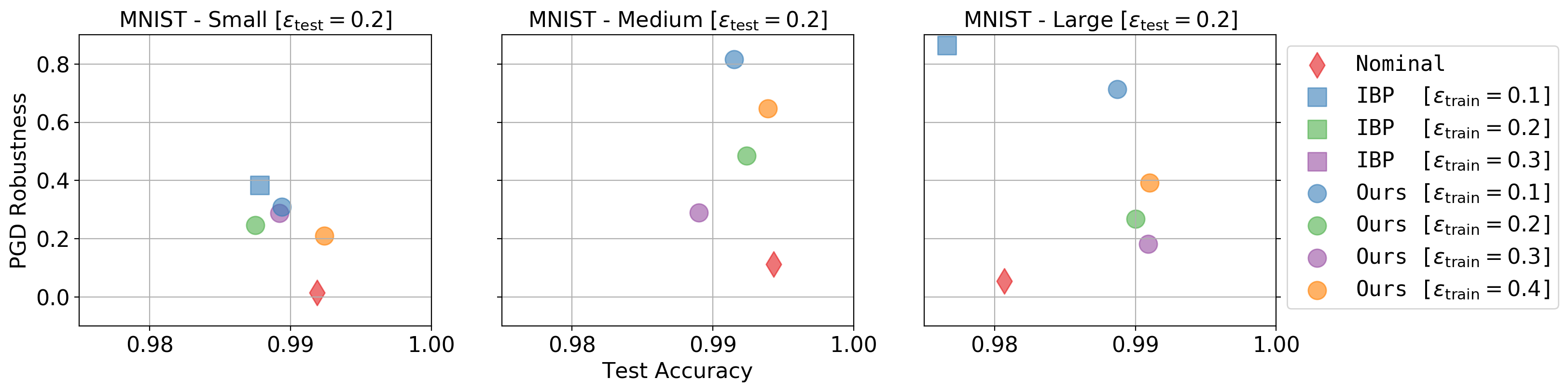}
    \includegraphics[width=\linewidth]{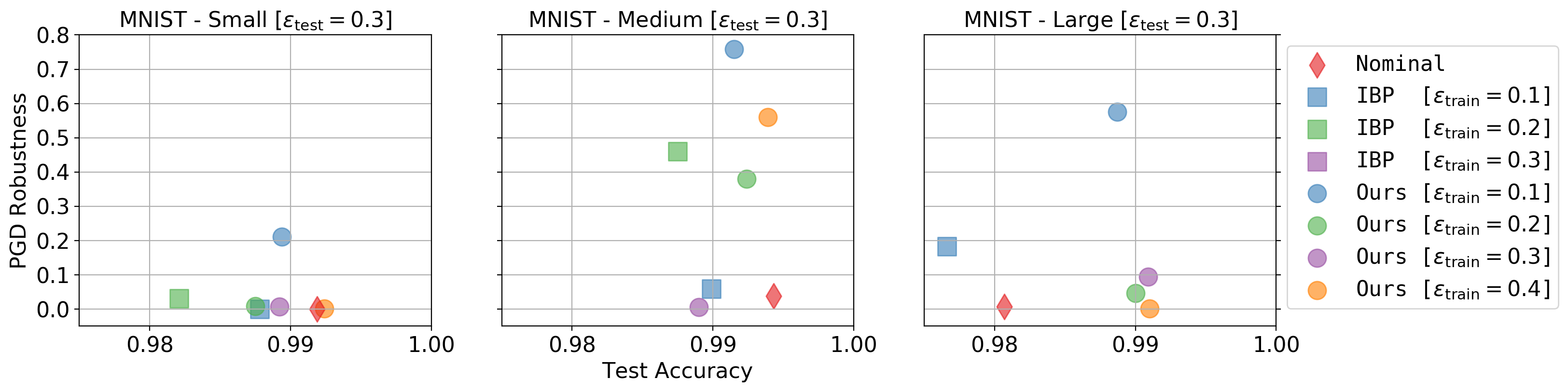}
    \includegraphics[width=\linewidth]{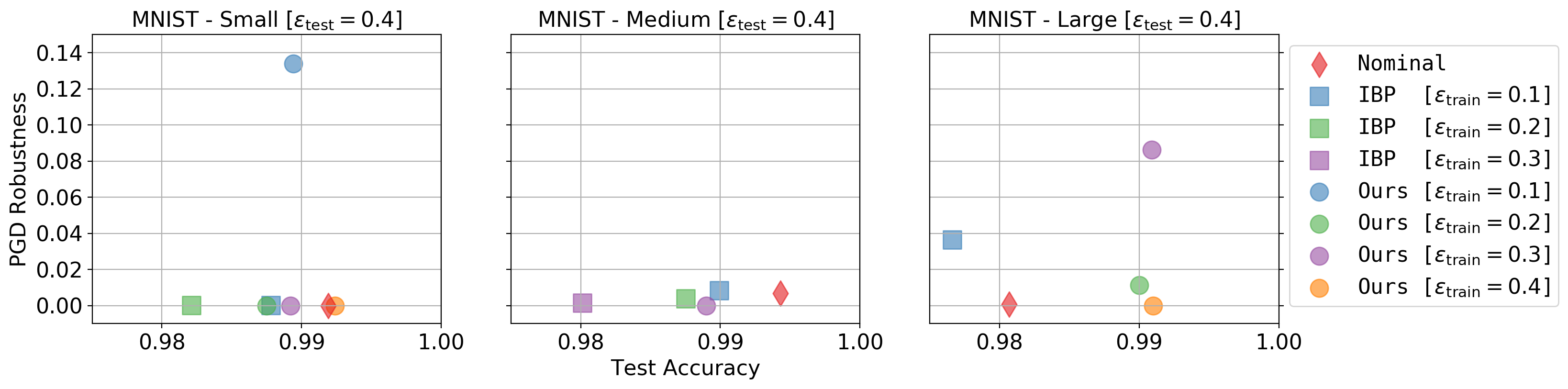}
    \caption{\textbf{PGD vs Accuracy on MNIST.} $\epsilon_\text{test} = 0.1,0.2,0.3$ and $0.4$, respectively.}\label{fig:mnist_train_1}
\end{figure*}

\begin{figure*}[h]
    \centering
    \includegraphics[width=\linewidth]{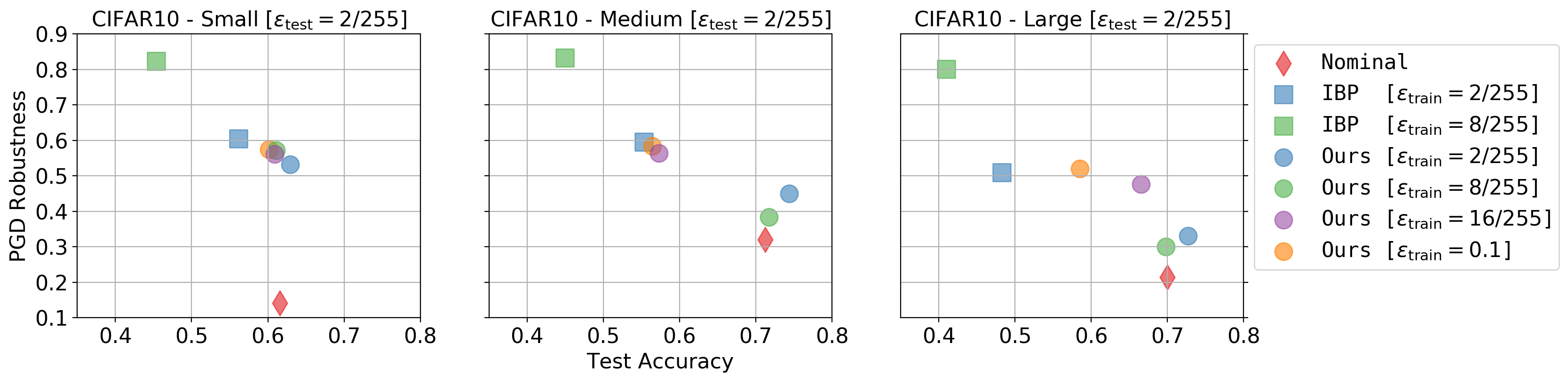}
    \includegraphics[width=\linewidth]{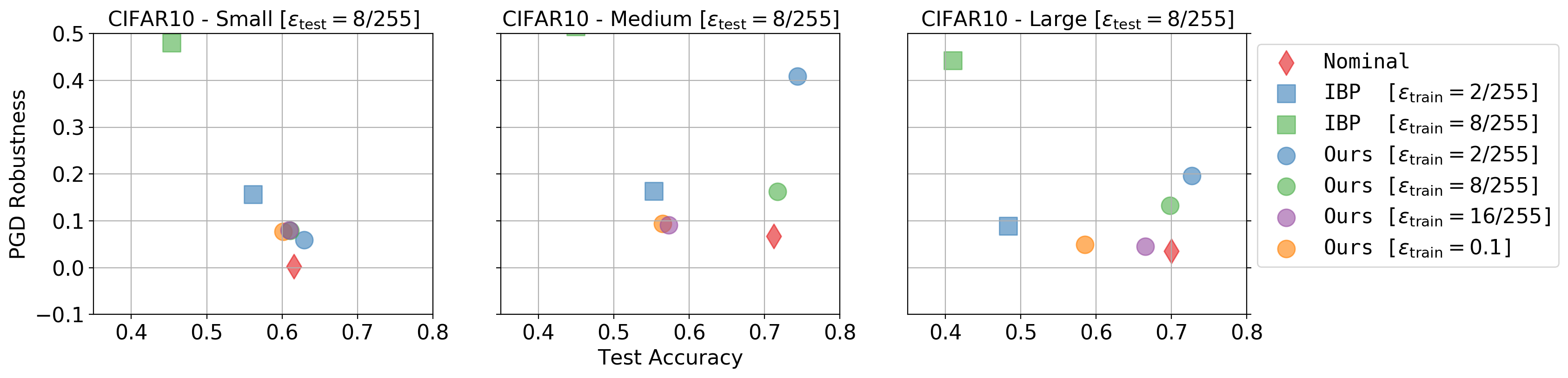}
    \includegraphics[width=\linewidth]{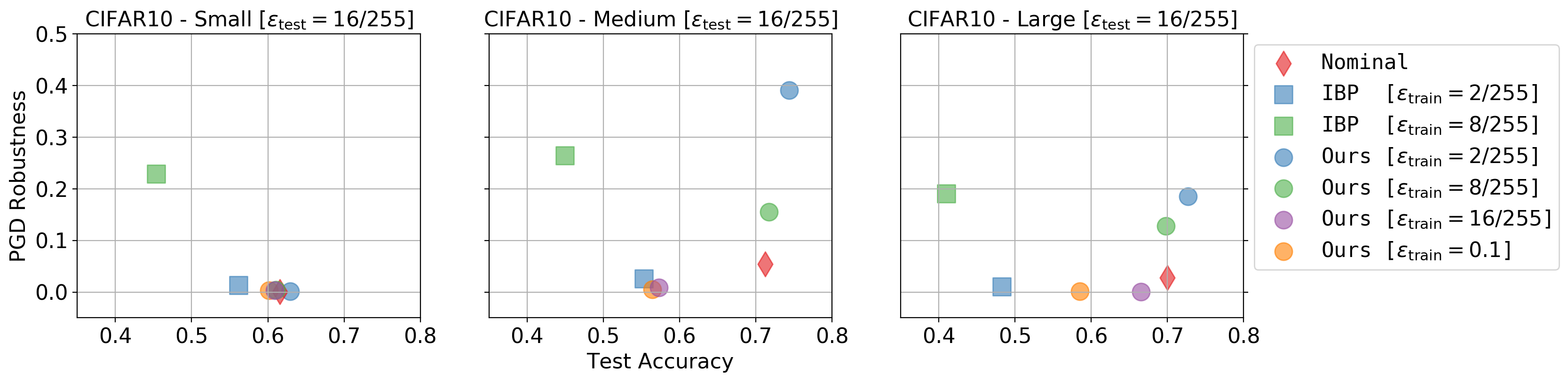}
    \includegraphics[width=\linewidth]{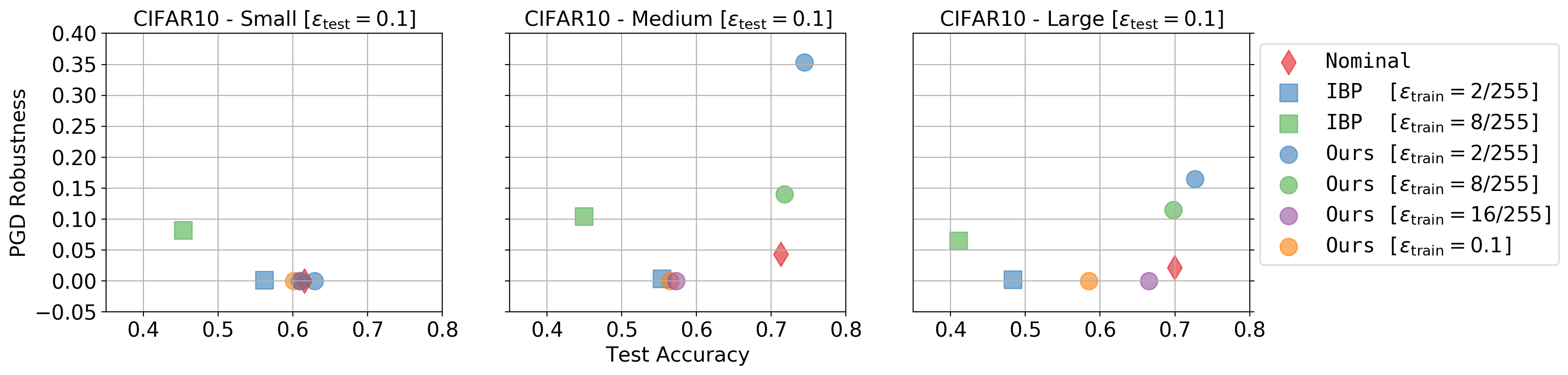}
    \caption{\textbf{PGD vs Accuracy on CIFAR10.} $\epsilon_\text{test} = \frac{2}{255}, \frac{8}{255}, \frac{16}{255}$ and $\frac{25.5}{255}$, respectively.}\label{fig:cifar10_train_2}
\end{figure*}
\clearpage

\section{ETB as Supersets in Expectation} \label{sec:superset_experiment}
Here, we validate Theorem \ref{thm:correctness_thm1} with several controlled experiments. For a network $g(\mathbf{x}) = \mathbf{a}_2^\top \max\left(\mathbf{A}_1\mathbf{x} + \mathbf{b}_1, \mathbf{0}_n\right) + b_2$ that has true bounds $[\mathbf{L}_{g}, \mathbf{U}_{g}]$ for $\tilde{\mathbf{x}} \in [\mathbf{x} -\epsilon\mathbf{1}_n, \mathbf{x} + \epsilon\mathbf{1}_n]$, we empirically show under the mild assumptions of Theorem \ref{thm:correctness_thm1} that $[\mathbf{L}_{\mathbf{ETB}}, \mathbf{U}_{\mathbf{ETB}}]$ is a superset to $[\mathbf{L}_{g},\mathbf{U}_{g}]$ in expectation.

Following the default PyTorch initialization \citep{paszke2017automatic}, we construct the biases. As for the elements of the weight matrices $\mathbf{A}_1 \in \mathbb{R}^{k \times n}$ and $\mathbf{A}_2 \in \mathbb{R}^{1 \times k}$, they are sampled from $\mathcal{N}(0,1/\sqrt{n})$ and $\mathcal{N}(0,1/\sqrt{k})$, respectively. We estimate $\mathbf{L}_{g}$ and $\mathbf{U}_{g}$ by taking the minimum and maximum of $10^{6} + 2^n$ Monte-Carlo evaluations of $g$. For a given $\mathbf{x}\sim\mathcal{N}(\mathbf{0}_n,\mathbf{I})$ and with $\epsilon = 0.1$, we uniformly sample $10^6$ examples from the interval $[\mathbf{x} - \epsilon\mathbf{1}_n,\mathbf{x}+\epsilon\mathbf{1}_n]$. We also sample all $2^n$ corners of the hyper cube $[\mathbf{x}-\epsilon\mathbf{1}_n,\mathbf{x}+\epsilon\mathbf{1}_n]$. To show that the proposed interval $[\mathbf{L}_{\mathbf{ETB}}, \mathbf{U}_{\mathbf{ETB}}]$ is a superset of $[\mathbf{L}_{g}, \mathbf{U}_{g}]$, we evaluate the length of the intersection of the two intervals over the length of the true interval defined as $\Gamma = |[\mathbf{L}_{\mathbf{ETB}},\mathbf{U}_{\mathbf{ETB}}] \cap [\mathbf{L}_{g},\mathbf{U}_{g}] |/|[\mathbf{L}_{g},\mathbf{U}_{g}]|$. Note that $\Gamma = 1$ if and only if $[\mathbf{L}_{\mathbf{ETB}}, \mathbf{U}_{\mathbf{ETB}}]$ is a superset to $[\mathbf{L}_{g}, \mathbf{U}_{g}]$. For a given $n$, we conduct this experiment $10^3$ times with varying $\mathbf{A}_1$, $\mathbf{A}_2$, $\mathbf{b}_1$, $\mathbf{b}_2$ and $\mathbf{x}$ and report the average $\Gamma$. Then, we run this for a varying number of input size $n$ and a varying number of hidden nodes $k$. As predicted by Theorem \ref{thm:correctness_thm1}, Figure \ref{fig:correctness_varying_input_sz} demonstrates that as $n$ increases, the proposed interval will be more likely to be a superset of the true interval, regardless of the number of hidden nodes $k$. Note that networks that are as wide as $k=1000$, require no more than $n=15$ input dimensions for the proposed intervals to be a superset of the true intervals. In practice, $n$ is much larger than that, \eg $n \approx 3 \times 10^3$ in CIFAR10.

In Figure \ref{fig:correctness_fc_varying_num_layers}, we empirically show that the above behavior persists in deeper networks. We propagate the bounds blockwise and conduct similar experiments on fully-connected networks. We vary the network depth but keep $k=n$ fixed. These results indeed suggest that the proposed bounds are supersets to the true bounds in expectation and are more likely so with larger $k$.

\begin{figure}[h]
    \centering
    \begin{subfigure}[t]{0.49\linewidth}
        \centering
        \includegraphics[width=\textwidth]{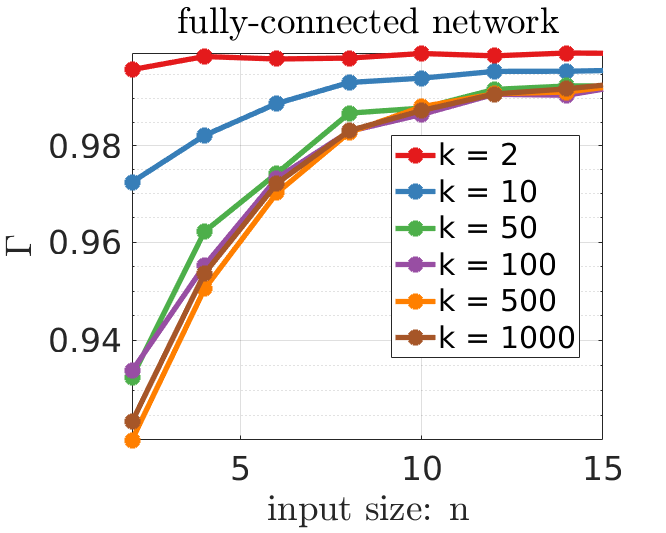}
        \caption{}\label{fig:correctness_varying_input_sz}
    \end{subfigure}%
    ~
    \begin{subfigure}[t]{0.49\linewidth}
        \centering
        \includegraphics[width=\textwidth]{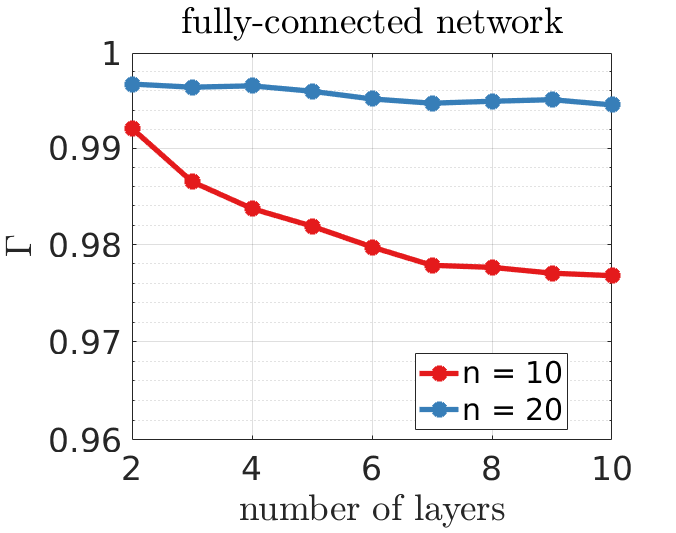}
        \caption{}\label{fig:correctness_fc_varying_num_layers}
    \end{subfigure}
    \caption{\textbf{ETB as supersets in expectation.} On the left, our proposed interval bounds $[\mathbf{L}_{\mathbf{ETB}},\mathbf{U}_{\mathbf{ETB}}]$, as predicted by Theorem \ref{thm:correctness_thm1}, get closer, with increasing $n$, to being a true superset to $[\mathbf{L}_{g},\mathbf{U}_{g}]$, which are estimated by Monte-Carlo Sampling, regardless of the number of hidden nodes $k$. On the right, we see a similar behavior under different depths.}
\end{figure}
\clearpage

\section{Tightness of ETB vs. IBP in Expectation}
In this experiment we focus our attention on the tightness of ETB when compared to IBP. In particular, we validate Theorem \ref{thm:tightness_thm2} by comparing the interval width of our proposed bounds, $W_{\mathbf{ETB}} = \mathbf{U}_{\mathbf{ETB}} - \mathbf{L}_{\mathbf{ETB}}$, with that of IBP, $W_{\mathbf{IBP}} = \mathbf{U}_{\mathbf{IBP}} - \mathbf{L}_{\mathbf{IBP}}$. We compute both the difference and ratio of widths for varying values of $k$, $n$, and $\epsilon$. Figure \ref{fig:tightness_2_layer_network} reports the average width difference and ratio over $10^3$ runs in a similar setup to the previous experiment. Figures \ref{fig:tightness_diff_varying_hidden_sz} and \ref{fig:tightness_ratio_varying_hidden_sz} show that the proposed bounds indeed get tighter than IBP, as $k$ increases across all $\epsilon$ values (as predicted by Theorem \ref{thm:tightness_thm2}). Note that we only show results for $\epsilon = \{0.01,0.1\}$ in Figure \ref{fig:tightness_ratio_varying_hidden_sz} as the performance of $\epsilon = \{0.5,1.0\}$ was very similar to $\epsilon=0.1$. Similar improvement occurs with increasing $n$, as in Figures \ref{fig:tightness_diff_varying_input_sz} and \ref{fig:tightness_ratio_varying_input_sz}.

We also compare the bounds under increasing depth for both fully-connected networks (refer to Figures \ref{fig:tightness_diff_fc_varying_depth} and \ref{fig:tightness_ratio_fc_varying_depth}) and convolutional networks (refer to Figures \ref{fig:tightness_diff_conv_varying_depth} and \ref{fig:tightness_ratio_conv_varying_depth}). For all fully-connected networks, we take $n=k=500$. Our proposed bounds get consistently tighter as the network depth increases over all choices of $\epsilon$. In particular, the proposed bounds can be more than $10^6$ times tighter than IBP for a ten-layer DNN. A similar observation can also be made for convolutional networks, where it is expensive to compute our bounds. So, instead, we obtain matrices $\mathbf{M}$ using the easy-to-compute IBP upper bounds. Despite this relaxation, we still obtain very tight expected bounds. Note that this slightly modified approach reduces exactly to our bounds for two-layer networks.

\begin{figure*}[h]
    \centering
    \begin{subfigure}[t]{0.23\textwidth}
        \centering
        \includegraphics[width=\textwidth]{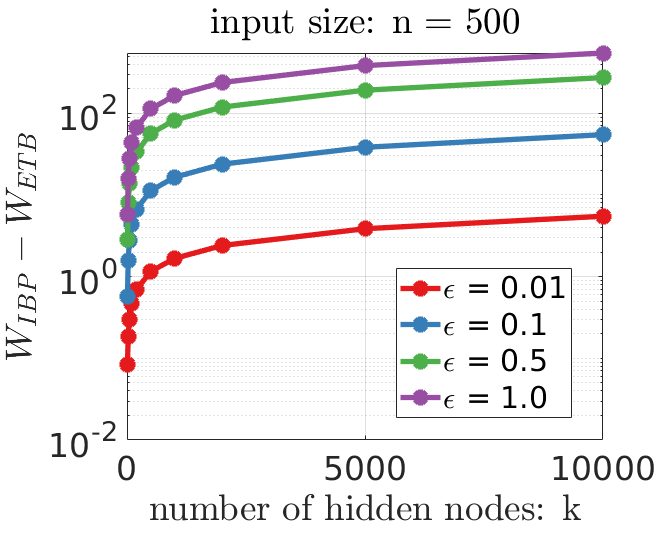}
        \caption{}\label{fig:tightness_diff_varying_hidden_sz}
    \end{subfigure}%
    ~
    \begin{subfigure}[t]{0.23\textwidth}
        \centering
        \includegraphics[width=\textwidth]{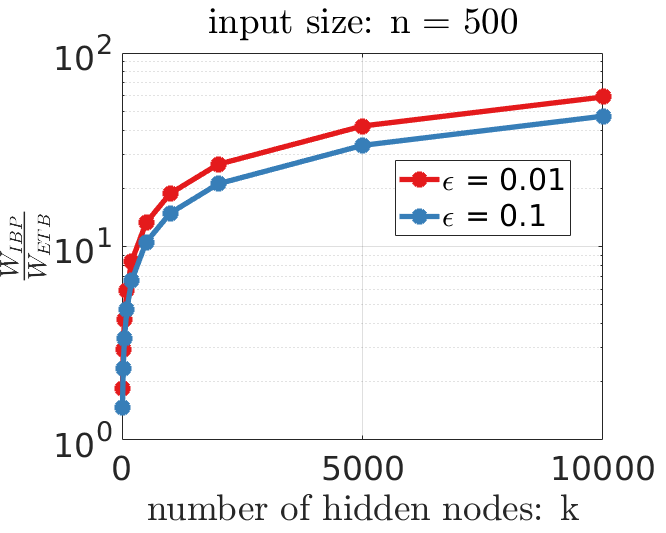}
        \caption{}\label{fig:tightness_ratio_varying_hidden_sz}
    \end{subfigure}
    ~
    \begin{subfigure}[t]{0.23\textwidth}
        \centering
        \includegraphics[width=\textwidth]{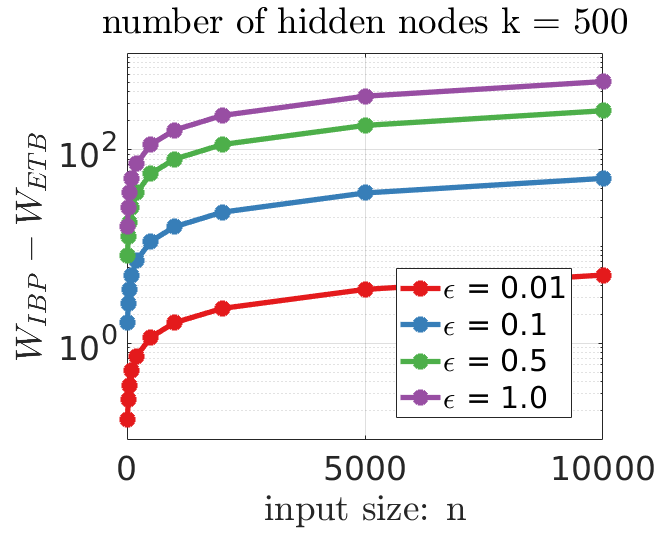}
        \caption{}\label{fig:tightness_diff_varying_input_sz}
    \end{subfigure}
    ~
    \begin{subfigure}[t]{0.23\textwidth}
        \centering
        \includegraphics[width=\textwidth]{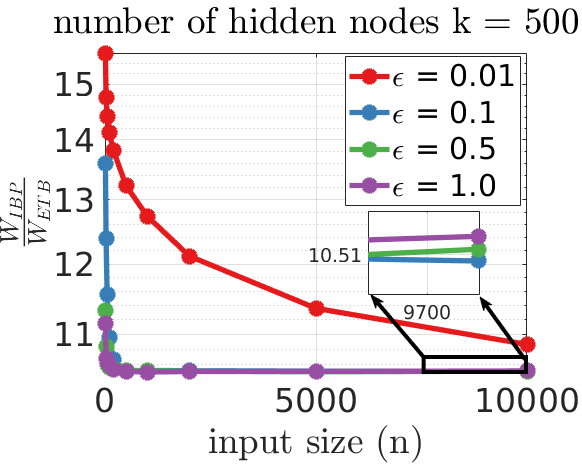}
        \caption{}\label{fig:tightness_ratio_varying_input_sz}
    \end{subfigure}
    \caption{\textbf{ETB is tighter than IBP under various input sizes and number of hidden nodes.} We show an interval bound tightness comparison between bounds of both ETB and IBP by comparing the difference and ratio of their interval lengths with varying $k$, $n$, and $\epsilon$ for a two-layer network. The proposed bounds are significantly tighter than IBP as predicted by Theorem \ref{thm:tightness_thm2}.}\label{fig:tightness_2_layer_network}
\end{figure*}

\begin{figure*}[h]
    \centering
    \begin{subfigure}[t]{0.23\textwidth}
        \centering
        \includegraphics[width=\textwidth]{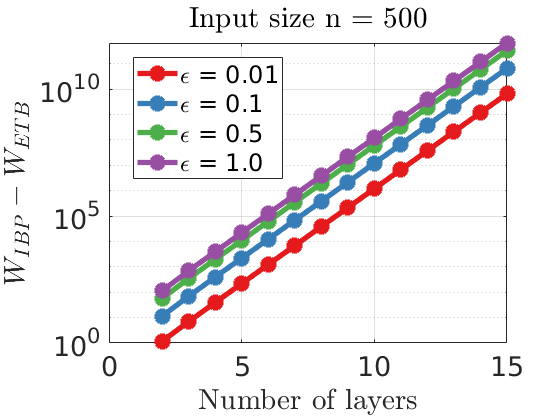}
        \caption{}\label{fig:tightness_diff_fc_varying_depth}
    \end{subfigure}%
    ~
    \begin{subfigure}[t]{0.23\textwidth}
        \centering
        \includegraphics[width=\textwidth]{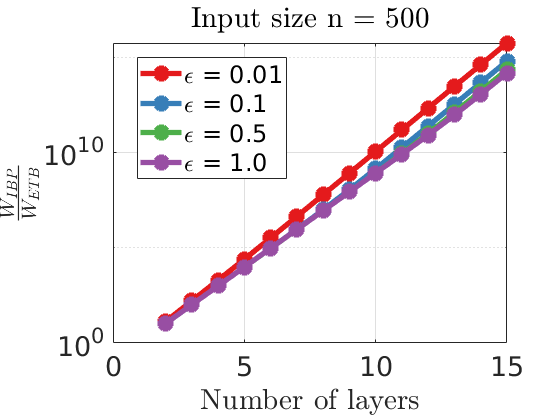}
        \caption{}\label{fig:tightness_ratio_fc_varying_depth}
    \end{subfigure}
    ~
    \begin{subfigure}[t]{0.23\textwidth}
        \centering
        \includegraphics[width=\textwidth]{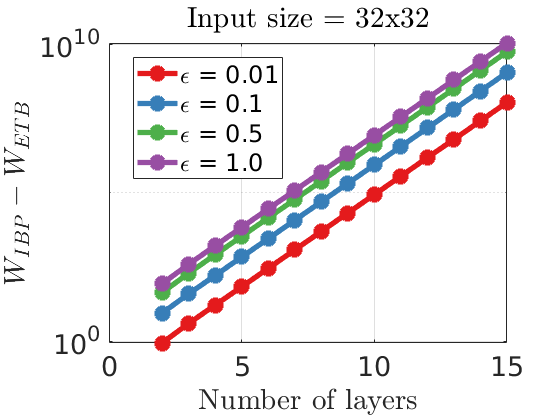}
        \caption{}\label{fig:tightness_diff_conv_varying_depth}
    \end{subfigure}
    ~
    \begin{subfigure}[t]{0.23\textwidth}
        \centering
        \includegraphics[width=\textwidth]{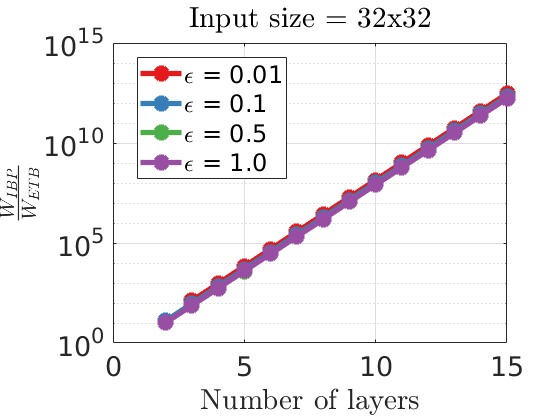}
        \caption{}\label{fig:tightness_ratio_conv_varying_depth}
    \end{subfigure}
    \caption{\textbf{ETB is tighter than IBP in deeper networks.} We show a bound tightness comparison between interval bounds computed by ETB against those of IBP by varying the number of layers for several choices of $\epsilon$. ETB is significantly tighter.}
\end{figure*}
\clearpage

\section{ETB vs. IBP on Real Networks}
We also train a three-layer network on the MNIST dataset with $\sim 99\%$ test accuracy. To show that ETB are also supersets to true bounds on real networks, and since the input dimension is too large for Monte-Carlo sampling ($n = 784$ pixels), we use the MIP formulation by \citet{tjeng19} with an identical parameter setting to \citet{gowal2018effectiveness}. We then report $\Gamma$ over varying testing $\epsilon$. Table \ref{table:correctness_mnist} demonstrates that indeed even on real networks beyond two layers and without the Gaussian weight assumption, our bounds are still supersets to the true bounds computed with the MIP solver and are much tighter than IBP. The results are averaged over $100$ randomly selected MNIST images.

\begin{table}[h]
    \centering
    \begin{tabular}{||c|c|c|c|c|c||}
        \hline
        $\epsilon$ & $\Gamma$         & $\Gamma_{\min}$ & $\Gamma_{\max}$ & $W_{\text{IBP}} - W_{\text{ETB}}$ & $\nicefrac{W_{\text{IBP}}}{W_{\text{ETB}}}$ \\
        \hline
        $0.01$     & $1.0 \pm 0$      & $1.0$           & $1.0$           & $644.322$                         & $17.391$                                    \\
        $0.02$     & $1.0 \pm 0$      & $1.0$           & $1.0$           & $1381.980$                        & $15.270$                                    \\
        $0.03$     & $0.97 \pm 0.088$ & $0.635$         & $1.0$           & $2255.397$                        & $14.555$                                    \\
        \hline
    \end{tabular}
    \caption{\textbf{ETB vs. IBP on Real Networks.} The table shows that our bounds are a superset to the true bounds, computed using an exact MIP solver, and much tighter than IBP.
    }\label{table:correctness_mnist}
\end{table}
\clearpage

\section{Qualitative Results} \label{sec:qualitative}
Following previous works \citep{kolter2017provable,gowal2018effectiveness}, we show qualitative results on synthetic random five-layer fully-connected networks with the architecture (\ie layer dimensions) $n$-$100$-$100$-$100$-$100$-$2$ where $n\in\{2,10,20\}$ is the input size. In Figure \ref{fig:qualitative_tables}, we visualize examples of the interval bounds of ETB and compare them to IBP and the true bounds estimated by Monte-Carlo sampling for several choices of $\epsilon \in \{0.05,0.1,0.25\}$.

\begin{figure}[!ht]
    \centering
    \begin{subfigure}[t]{\textwidth}
        \centering
        \begin{minipage}{0.50\textwidth}
            \includegraphics[width=\textwidth]{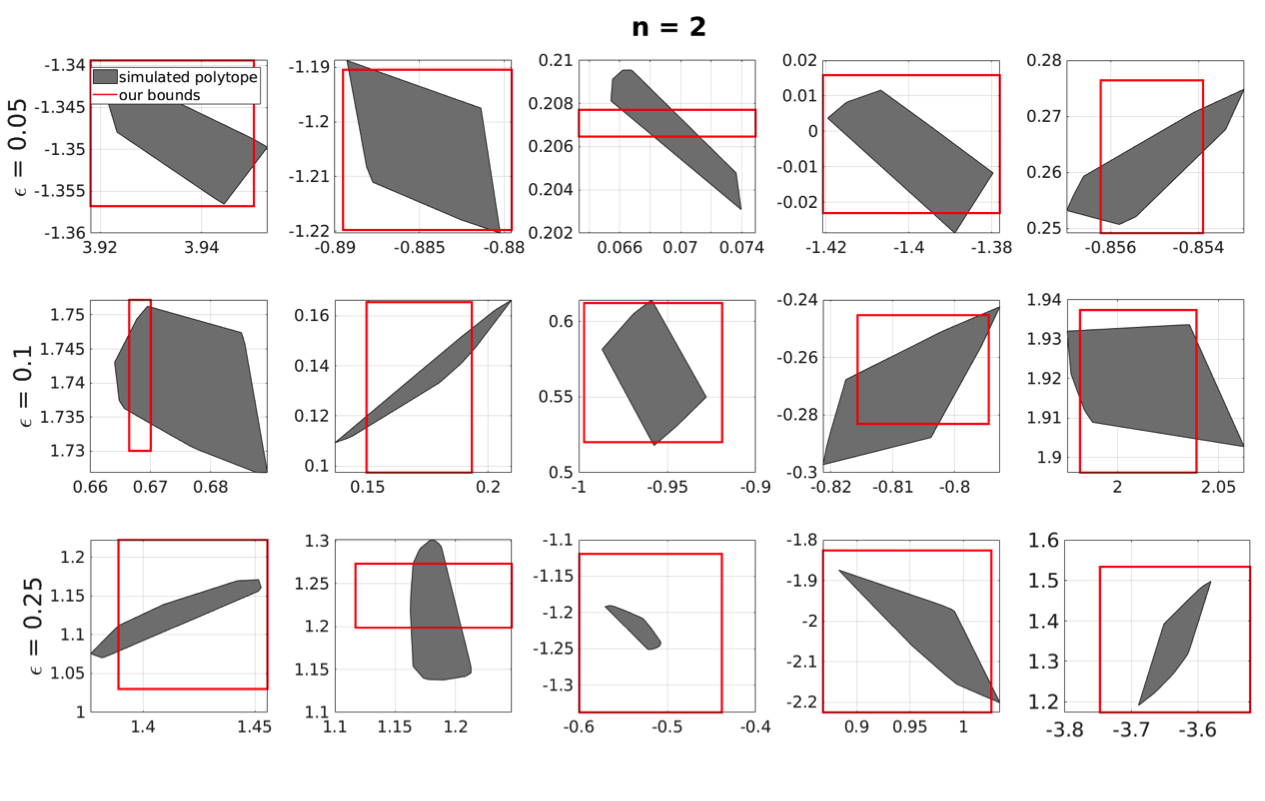}
        \end{minipage}
        \begin{minipage}{0.39\textwidth}
            \scalebox{0.63}{
                \begin{tabular}{c|c|c|c|c}
                    \toprule
                    $\epsilon$, column           & $\mathbf{L}_{\text{IBP}}^x$ & $\mathbf{U}_{\text{IBP}}^x$ & $\mathbf{L}_{\text{IBP}}^y$ & $\mathbf{U}_{\text{IBP}}^y$ \\
                    \midrule
                    $\epsilon=0.05$, column $=1$ & -10.1261          & 19.0773           & -18.1500          & 13.3573           \\
                    $\epsilon=0.05$, column $=2$ & -12.2529          & 14.3428           & -14.4295          & 12.3479           \\
                    $\epsilon=0.05$, column $=3$ & -12.6594          & 14.1837           & -12.5873          & 12.2612           \\
                    $\epsilon=0.05$, column $=4$ & -17.7825          & 16.4048           & -15.3843          & 15.1688           \\
                    $\epsilon=0.05$, column $=5$ & -12.5260          & 11.1149           & -8.9242           & 12.7539           \\
                    \bottomrule
                    \bottomrule
                    $\epsilon=0.1$, column $=1$  & -27.4598          & 23.6603           & -17.9481          & 23.4817           \\
                    $\epsilon=0.1$, column $=2$  & -23.2877          & 34.0542           & -28.1535          & 21.8703           \\
                    $\epsilon=0.1$, column $=3$  & -35.2950          & 36.4901           & -31.7465          & 36.0421           \\
                    $\epsilon=0.1$, column $=4$  & -31.7154          & 29.3062           & -30.3900          & 35.7105           \\
                    $\epsilon=0.1$, column $=5$  & -25.0870          & 39.4373           & -24.5087          & 32.5493           \\
                    \bottomrule
                    \bottomrule
                    $\epsilon=0.25$, column $=1$ & -54.0557          & 56.2884           & -52.5686          & 73.9621           \\
                    $\epsilon=0.25$, column $=2$ & -59.2115          & 82.8742           & -75.7999
                                                        & 65.6898                                                                       \\
                    $\epsilon=0.25$, column $=3$ & -50.1142          & 56.2330           & -72.4221          & 54.4631           \\
                    $\epsilon=0.25$, column $=4$ & -52.6030          & 83.3950           & -92.8100          & 69.1401           \\
                    $\epsilon=0.25$, column $=5$ & -89.1335          & 43.4685           & -74.4519          & 91.5137           \\
                    \bottomrule
                \end{tabular}
            }
        \end{minipage}
        \caption{\textbf{When $n = 1$,} the proposed bounds are far from being true this is as predicted by Theorem \ref{thm:correctness_thm1} for small $n$.}
    \end{subfigure}
    \begin{subfigure}[t]{\textwidth}
        \centering
        \begin{minipage}{0.50\textwidth}
            \includegraphics[width=\textwidth]{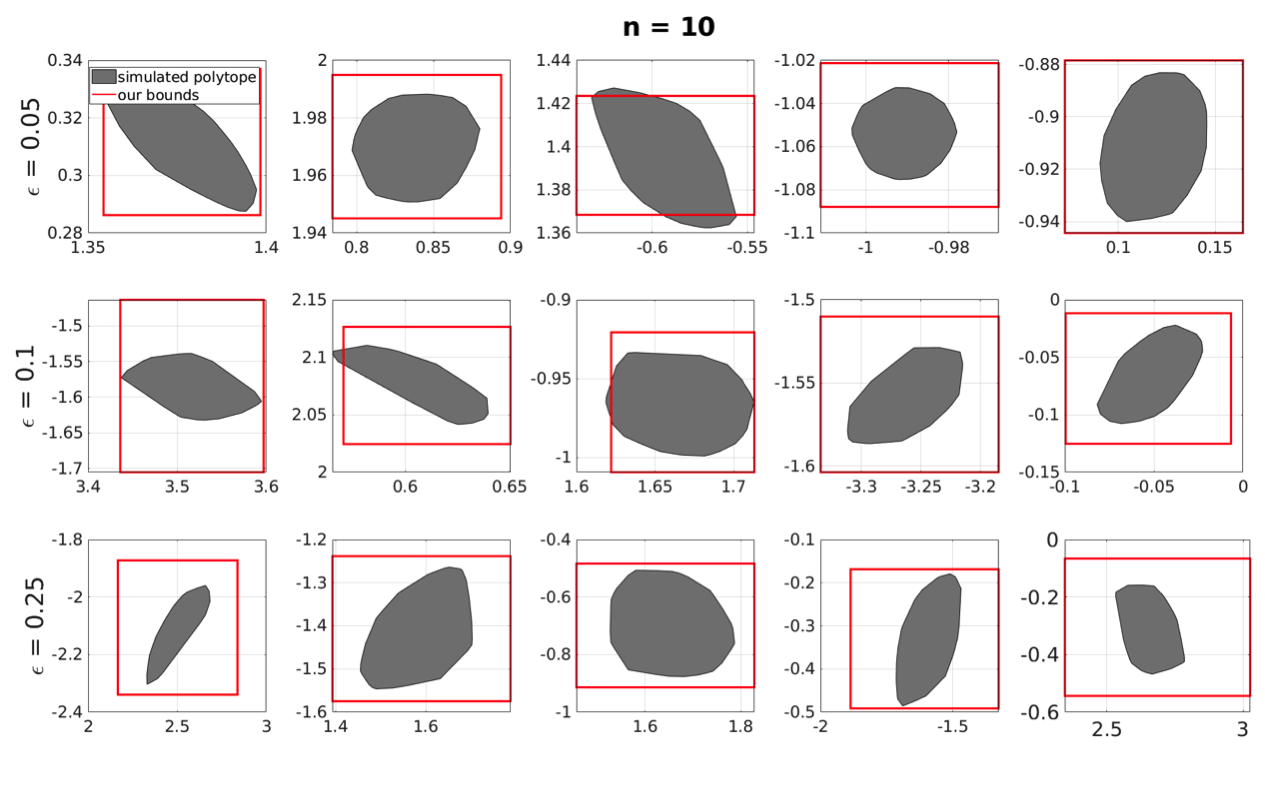}
        \end{minipage}
        \begin{minipage}{0.39\textwidth}
            \scalebox{0.63}{
                \begin{tabular}{c|c|c|c|c}
                    \toprule
                    $\epsilon$, column           & $\mathbf{L}_{\text{IBP}}^x$ & $\mathbf{U}_{\text{IBP}}^x$ & $\mathbf{L}_{\text{IBP}}^y$ & $\mathbf{U}_{\text{IBP}}^y$ \\
                    \midrule
                    $\epsilon=0.05$, column $=1$ & -16.9716          & 24.9259           & -21.2584          & 20.6358           \\
                    $\epsilon=0.05$, column $=2$ & -42.6267          & 48.1786           & -38.6958          & 37.7851           \\
                    $\epsilon=0.05$, column $=3$ & -41.4147          & 36.4056           & -42.0363          & 36.6605           \\
                    $\epsilon=0.05$, column $=4$ & -32.1013          & 25.1485           & -37.7864          & 33.3652           \\
                    $\epsilon=0.05$, column $=5$ & -45.4368          & 32.9774           & -44.8946          & 38.6805           \\
                    \bottomrule
                    \bottomrule
                    $\epsilon=0.1$, column $=1$  & -48.1221          & 86.6800           & -54.3059          & 71.2724           \\
                    $\epsilon=0.1$, column $=2$  & -51.2668          & 46.1237           & -38.8089
                                                        & 33.6512                                                                       \\
                    $\epsilon=0.1$, column $=3$  & -51.3915          & 52.4437           & -52.7149          & 49.1031           \\
                    $\epsilon=0.1$, column $=4$  & -71.7738          & 54.4836           & -91.0335          & 37.0950           \\
                    $\epsilon=0.1$, column $=5$  & -48.1744          & 33.2927           & -40.9540          & 47.2282           \\
                    \bottomrule
                    \bottomrule
                    $\epsilon=0.25$, column $=1$ & -152.7639         & 192.4156          & -188.4030         & 148.2482          \\
                    $\epsilon=0.25$, column $=2$ & -196.8923         & 195.4355          & -163.2691         & 177.3766          \\
                    $\epsilon=0.25$, column $=3$ & -141.6800         & 207.5414          & -207.9396         & 190.2823          \\
                    $\epsilon=0.25$, column $=4$ & -200.7513         & 156.2560          & -227.6427         & 182.0180          \\
                    $\epsilon=0.25$, column $=5$ & -153.3898         & 164.8314          & -147.8662         & 137.4380          \\
                    \bottomrule
                \end{tabular}
            }
        \end{minipage}
        \caption{\textbf{When $n = 10$,} the bounds are more likely now to enclose the true output region for all given $\epsilon$.}
    \end{subfigure}
    \begin{subfigure}[t]{\textwidth}
        \centering
        \begin{minipage}{0.50\textwidth}
            \includegraphics[width=\textwidth]{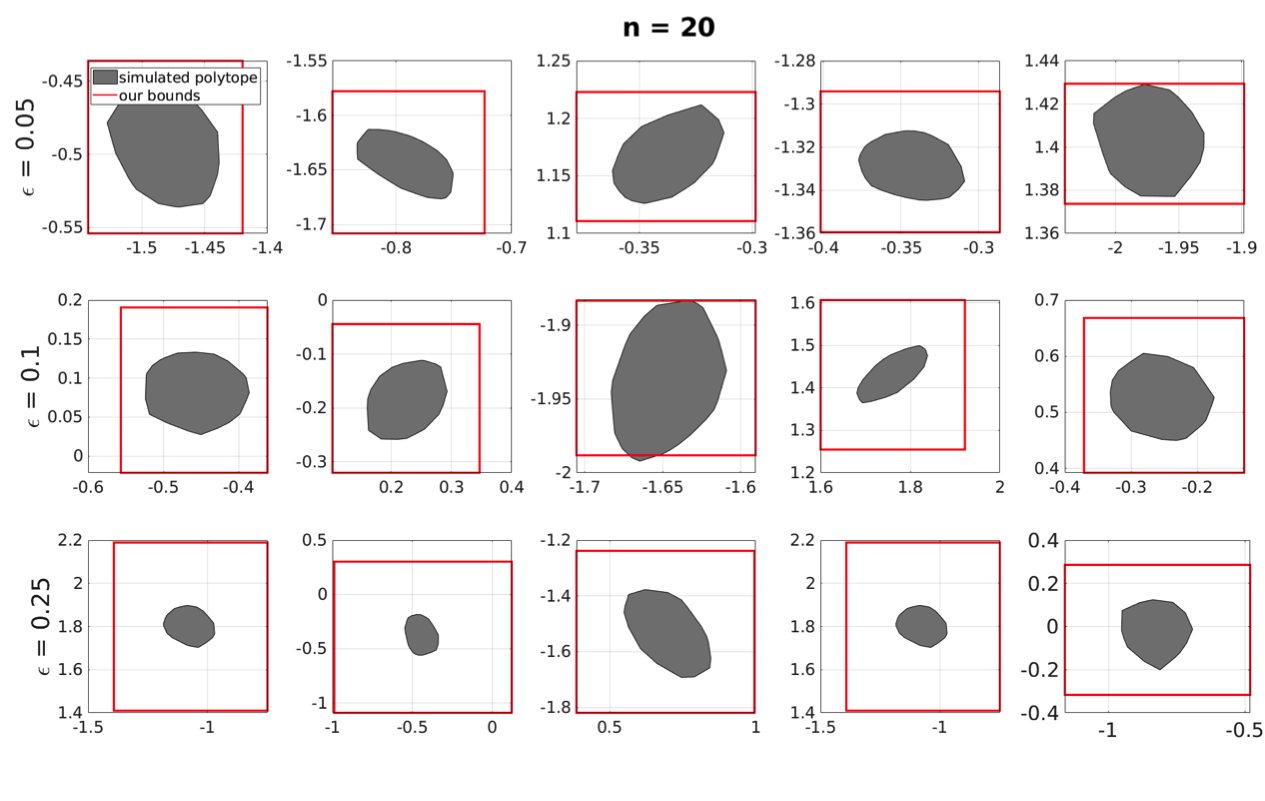}
        \end{minipage}
        \begin{minipage}{0.39\textwidth}
            \scalebox{0.63}{
                \begin{tabular}{c|c|c|c|c}
                    \toprule
                    $\epsilon$, column           & $\mathbf{L}_{\text{IBP}}^x$ & $\mathbf{U}_{\text{IBP}}^x$ & $\mathbf{L}_{\text{IBP}}^y$ & $\mathbf{U}_{\text{IBP}}^y$ \\
                    \midrule
                    $\epsilon=0.05$, column $=1$ & -43.6689          & 32.8572           & -47.7856          & 36.9842           \\
                    $\epsilon=0.05$, column $=2$ & -53.2447          & 47.1651           & -46.5306          & 53.1638           \\
                    $\epsilon=0.05$, column $=3$ & -59.1694          & 42.6647           & -43.4659          & 57.2781           \\
                    $\epsilon=0.05$, column $=4$ & -39.8479          & 42.4197           & -42.1962          & 39.7649           \\
                    $\epsilon=0.05$, column $=5$ & -54.3150          & 42.8637           & -44.5742          & 43.8117           \\
                    \bottomrule
                    \bottomrule
                    $\epsilon=0.1$, column $=1$  & -83.5804          & 81.9034           & -97.5203          & 98.8713           \\
                    $\epsilon=0.1$, column $=2$  & -64.8464          & 76.8083           & -84.9223          & 83.9505           \\
                    $\epsilon=0.1$, column $=3$  & -70.5862          & 92.6652           & -88.6098          & 71.8915           \\
                    $\epsilon=0.1$, column $=4$  & -78.0557          & 151.4360          & -106.073          & 123.3686          \\
                    $\epsilon=0.1$, column $=5$  & -91.8368          & 97.3438           & -103.2845         & 76.6581           \\
                    \bottomrule
                    \bottomrule
                    $\epsilon=0.25$, column $=1$ & -188.7623         & 256.2275          & -211.3972         & 255.5101          \\
                    $\epsilon=0.25$, column $=2$ & -219.5642         & 274.5287          & -217.7622         & 349.4256          \\
                    $\epsilon=0.25$, column $=3$ & -214.7457         & 160.7498          & -186.5554         & 184.1767          \\
                    $\epsilon=0.25$, column $=4$ & -188.7623         & 256.2275          & -211.3972         & 255.5101          \\
                    $\epsilon=0.25$, column $=5$ & -276.9137         & 177.7929          & -202.2031         & 245.8731          \\
                    \bottomrule
                \end{tabular}
            }
        \end{minipage}
        \caption{\textbf{When $n = 20$,} the bounds almost always enclose the polytope.}
    \end{subfigure}
    \caption{\textbf{Qualitative Results.} We show our ETB $[\mathbf{L}_{\text{M}},\mathbf{U}_{\text{M}}]$ in red. In every block, each row represents a choice of $\epsilon$ with 5 different randomly initialized networks. Note that, as predicted by Theorem \ref{thm:correctness_thm1}, ETB improve significantly (in becoming supersets to the true bounds) as $n$ increases. Moreover, they are significantly tighter than IBP (tables on the right) as predicted by Theorem \ref{thm:tightness_thm2}.}\label{fig:qualitative_tables}
\end{figure}
\clearpage

\section{Comments on Assumption \ref{key_assumption}}

\subsection{Hypothetical Failure Example}
Generally speaking, our proposed bounds $[\mathbf{L}_{\mathbf{ETB}}, \mathbf{U}_{\mathbf{ETB}}]$ computed using Equation (\ref{eq:etb}) can be very loose if the network weights $\mathbf{A_1}, \mathbf{a}_2$ do not follow Assumption \ref{key_assumption}. In this section, we show a failure example when this Gaussian premise is violated. Consider a two-layer network $g(\mathbf{x})$ where $\mathbf{A}_1 = 1000~\mathbf{I}_{n \times n}$, $\mathbf{b}_1 = -999~\mathbf{1}_n$, $\mathbf{a}_2 = -10~\mathbf{1}_{n}$, and $b_2 = 0$. For the interval $[-\mathbf{1}_n , \mathbf{1}_n]$, the output lower and upper bounds of the first layer are given as

\begin{equation*}
    \mathbf{L},\mathbf{U}
    = \mathbf{b}_1 \mp \epsilon |\mathbf{A}_1| \mathbf{1}_n
    = \mathbf{1}_n \Big(\mp1000  - 999\Big) = -1999~\mathbf{1}_n,\mathbf{1}_n.
\end{equation*}

Since $\mathbf{U} \ge \mathbf{0}$, then $\mathbf{M} = \mathbf{I}_{n\times n}$. Therefore, our estimated bounds by Equation (\ref{eq:etb}) are

\begin{equation*}
    \mathbf{L}_{\mathbf{ETB}}, \mathbf{U}_{\mathbf{ETB}}
    = \mathbf{a}_2^\top\mathbf{b}_1 \mp \epsilon |\mathbf{a}_2^\top \mathbf{A}_1|\mathbf{1}_n
    = 9990 n \mp 10000~\epsilon \mathbf{1}_n^\top \mathbf{1}_n = -10n, 19990n.
\end{equation*}

As for IBP bounds, they are given as follows:

\begin{align*}
    \mathbf{L}_{\mathbf{IBP}}, \mathbf{U}_{\mathbf{IBP}} & = \mathbf{a}_2^\top \Big(\frac{\max(\mathbf{U},\mathbf{0}_n)  + \max(\mathbf{L},\mathbf{0})}{2}\Big) \mp |\mathbf{a}_2^\top| \Big(\frac{\max(\mathbf{U},\mathbf{0}_n) - \max(\mathbf{L},\mathbf{0})}{2}\Big) \\
                                                         & = -\frac{10}{2} \mathbf{1}_n^\top \mathbf{1}_n \mp \frac{10}{2} \mathbf{1}_n^\top \mathbf{1}_n = -10n, 0.
\end{align*}

Under this construction of network weights, violating the i.i.d Gaussian assumption, it is clear that our bounds can be orders of magnitude looser to IBP. In the following section, we demonstrate that networks trained on real data do have weights that are not far from the Gaussian assumption by empirically investigating the histogram of the network weights.

\begin{figure}[!ht]
    \centering
    \begin{subfigure}[t]{\textwidth}
        \centering
        \includegraphics[width=0.49\textwidth]{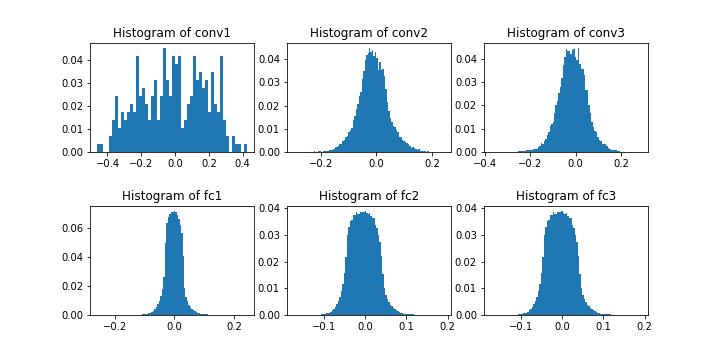}
        \includegraphics[width=0.49\textwidth]{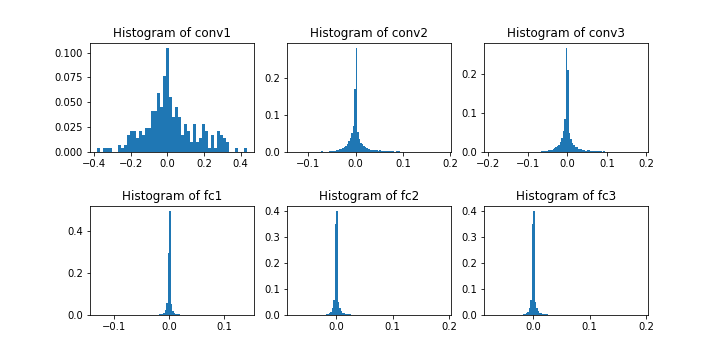}
        \caption{MNIST}\label{fig:mnist_gauss_1_main}
    \end{subfigure}
    \begin{subfigure}[t]{\textwidth}
        \centering
        \includegraphics[width=0.49\textwidth]{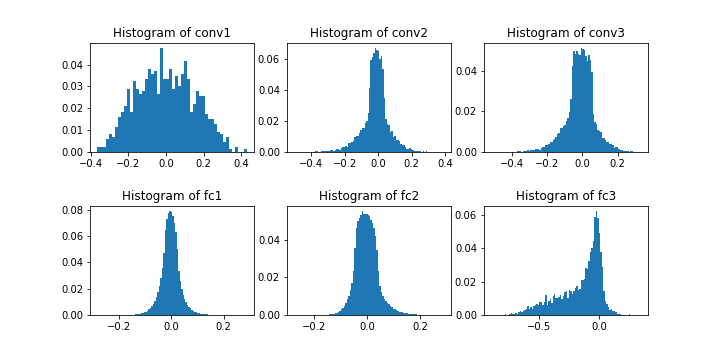}
        \includegraphics[width=0.49\textwidth]{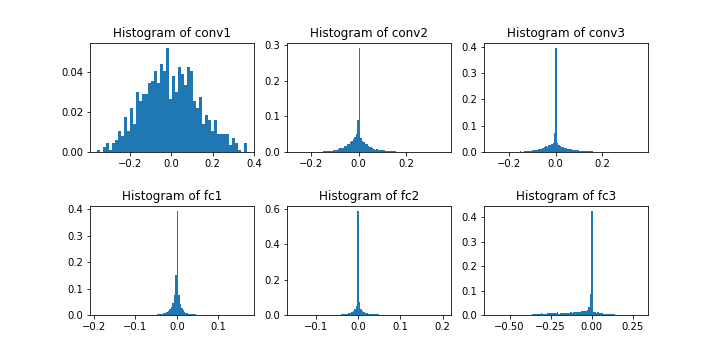}
        \caption{CIFAR10}\label{fig:cifar10_gauss_1_main}
    \end{subfigure}
    \begin{subfigure}[t]{\textwidth}
        \centering
        \includegraphics[width=0.49\textwidth]{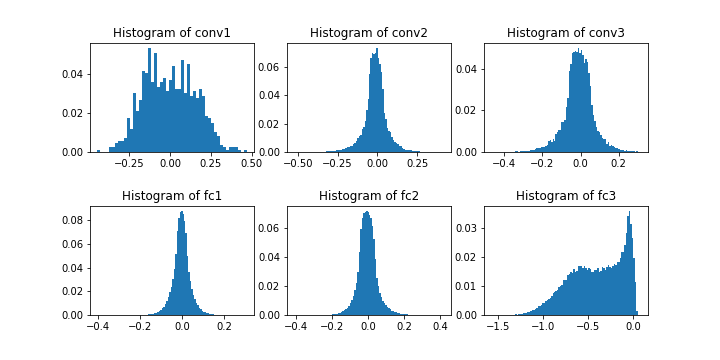}
        \includegraphics[width=0.49\textwidth]{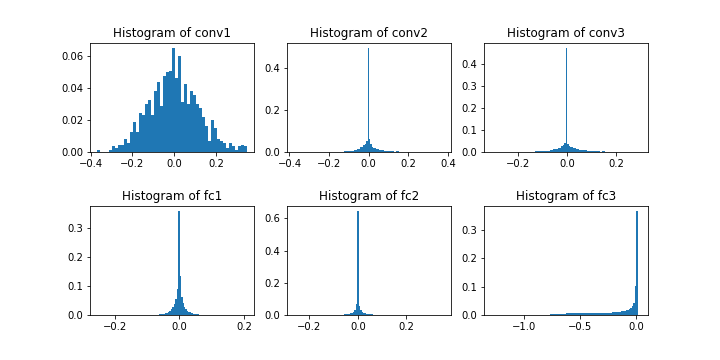}
        \caption{CIFAR100}\label{fig:cifar100_gauss_1_main}
    \end{subfigure}
    \caption{\textbf{Histograms of the network weights trained with and without $\ell_2$ regularization.} The histograms are for the first 6 layers of a medium sized network trained on MNIST, CIFAR10, and CIFAR100. The figures on the left and right are for networks trained without and with $\ell_2$ regularization, respectively.}
\end{figure}

\subsection{Assumption Validity on Real Networks}
While the Gaussian i.i.d. assumption can be strong for general networks trained on real data, it is not far from being reasonable due to commonly accepted training procedures. This is since it is common to regularize network weights while training with an $\ell_2$ regularizer encouraging the weights to follow a zero-mean Gaussian distribution let alone that networks in many cases are initialized in such manner. However, to quantify how reasonable is the assumption of Gaussian i.i.d. weights in trained networks, we visualize the histogram of the weights of a trained medium-sized network (provided by \citet{gowal2018effectiveness}) on MNIST in Figure \ref{fig:mnist_gauss_1_main}. As can be seen, all the histograms of the weights have a center of mass around 0 with a bell shaped looking distribution. A similar observation can be noted for CIFAR10 in Figure \ref{fig:cifar10_gauss_1_main} and CIFAR100 in Figure \ref{fig:cifar100_gauss_1_main}. This provide some evidence that the assumption of standard Gaussian is not far from being realistic.
\clearpage

\section{Proofs}
\begin{proposition}
    For $\mathbf{a} \in \mathbb{R}^n \sim \mathcal{N}(\mathbf{0},\sigma_a^2 \mathbf{I})$ and a uniform random vector $\tilde{\mathbf{x}} \sim \mathcal{U}[\mathbf{x} - \epsilon \mathbf{1}_n,\mathbf{x} + \epsilon\mathbf{1}_n]$ where both $\mathbf{a}$ and $\tilde{\mathbf{x}}$ are independent, we have that Lyapunov Central Limit Theorem holds such that
    \begin{align}
        \frac{1}{s_n} \sum_{i=1}^n (\tilde{x}_i a_i - \mathbb{E}[a_i \tilde{x}_i]) \rightarrow^d \mathcal{N}(0,1), ~~ \text{where} ~~ s_n^2 = \text{Var}\left(\sum_{i=1}^n (\tilde{x}_i a_i - \mathbb{E}[a_i \tilde{x}_i]) \right) \nonumber
    \end{align}
    where $\rightarrow^d$ indicates convergence in distribution.
\end{proposition}

\proof{
    The Lyapunov condition
    \begin{align}
        \exists \delta >0, \frac{1}{s_n^{2+\delta}} \sum_{i=1}^n \mathbb{E}\left[\big| \tilde{x}_i a_i - \mathbb{E}[a_i \tilde{x}_i] \big|^{2+\delta}\right] \rightarrow 0, \text{as} ~~ n \rightarrow \infty
    \end{align}
    is sufficient for Lyapunov Central Limit Theorem to hold.  Note that
    \begin{align}
        s_n^2 = \sum_{i=1}^2 \text{Var}\left(\tilde{x}_i a_i\right) = \sum_{i=1}^n \mathbb{E}\left[a_i^2 \tilde{x}_i^2\right] = \sigma_a^2 \sum_{i=1}^n \left(\frac{\epsilon^2}{3} + x_i^2\right) = \sigma_a^2 \left(\frac{n\epsilon^2}{3} + \sum_{i=1}^n x_i^2\right).
    \end{align}
    Since for $\delta = 2$, we have that
    \begin{align*}
        \mathbb{E}[|a_i \tilde{x}_i|^{2+\delta}] & = \int_{-\infty}^{\infty} \int_{x_i - \epsilon}^{x_i + \epsilon} a_i^2 \tilde{x}_i^2 \frac{1}{2 \epsilon} \frac{1}{\sqrt{2 \pi} \sigma_a} \exp\left(- \frac{a_i^2}{2 \sigma_a^2}\right) d a_i d \tilde{x}_i \\
                                                 & =\frac{3 \sigma_a^4}{2 \epsilon} \int_{x_i - \epsilon}^{x_i + \epsilon}  \tilde{x}_i^4  d\tilde{x}_i = \frac{3 \sigma_a^4}{10 \epsilon} \left[(x_i + \epsilon)^5 - (x_i - \epsilon)^5\right].
    \end{align*}
    Thereafter, Lyapunov Central Limit Theorem with $\delta = 2$ is satisfied since
    \begin{align*}
        \lim_{n \to \infty}\frac{1}{s_n^{4}} \sum_{i=1}^n \mathbb{E}\left[\big| \tilde{x}_i a_i  -    \mathbb{E}[a_i \tilde{x}_i]  \big|^{2+\delta}\right] & =   \lim_{n \to \infty} \frac{1}{s_n^{4}} \sum_{i=1}^n \mathbb{E}\left[\big| \tilde{x}_i a_i \big|^{4}\right]                                                                                  \\
                                                                                                                                                           & =  \lim_{n \to \infty} \frac{3  \sum_{i=1}^n (x_i + \epsilon)^5 - (x_i - \epsilon)^5}{10 \epsilon \left(\frac{n \epsilon^2}{3} + \sum_{i=1}^n x_i^2\right)^2}                                  \\
                                                                                                                                                           & \leq    \lim_{n \to \infty} \frac{3  n \left((x_{\text{max}} + \epsilon)^5 - (x_{\text{min}} - \epsilon)^5\right)}{10 \epsilon \left(\frac{n \epsilon^2}{3} + \sum_{i=1}^n x_i^2\right)^2} = 0 \\
    \end{align*}
}

\begin{proposition}
    A random matrix $\mathbf{A}_1$ with i.i.d. Gaussian elements of zero mean and $\sigma_{\mathbf{A}_1}$ standard deviation and a uniform random vector $\tilde{\mathbf{x}} \sim \mathcal{U} \left[\mathbf{x} - \epsilon \mathbf{1}_n, \mathbf{x} + \epsilon\mathbf{1}_n\right]$ we have that $\text{Covariance}\left(\mathbf{A}\tilde{\mathbf{x}}\right) = \left(\frac{\epsilon^2 \sigma_{\mathbf{A}_1}^2 n}{3} +  \sigma_{\mathbf{A}_1}^2 \text{trace}\left(\mathbf{x} \mathbf{x}^\top\right)\right) \mathbf{I}$.
\end{proposition}

\proof{
    The former follows from the fact that
    \begin{align*}
        \text{Covariance}\left(\mathbf{A}_1\tilde{\mathbf{x}} + \mathbf{b}_1\right) & = \text{Covariance}\left(\mathbf{A}_1 \tilde{\mathbf{x}}\right)                                                                                                                        \\
                                                                                    & = \mathbb{E}\left[\mathbf{A}_1 \tilde{\mathbf{x}} \tilde{\mathbf{x}}^\top\mathbf{A}_1^\top\right]                                                                                      \\
                                                                                    & = \mathbb{E}_{\mathbf{A}_1} \left[\mathbf{A}_1 \mathbb{E}\left[\tilde{\mathbf{x}}\tilde{\mathbf{x}}^\top\right]\mathbf{A}_1^\top\right]                                                \\
                                                                                    & = \mathbb{E}_{\mathbf{A}_1} \left[\mathbf{A}_1 \left(\text{Diag}\left(\frac{\epsilon^2}{3}\right) + \mathbf{x} \mathbf{x}^\top \right) \mathbf{A}_1^\top\right]                        \\
                                                                                    & = \frac{\epsilon^2}{3}\mathbb{E}_{\mathbf{A}_1} \left[\mathbf{A}_1 \mathbf{A}_1^\top\right] + \mathbb{E}\left[\mathbf{A}_1 \mathbf{x} \mathbf{x}^\top \mathbf{A}_1^\top\right]         \\
                                                                                    & =  \frac{\epsilon^2}{3}\mathbb{E}_{\mathbf{A}_1} \left[\mathbf{A}_1 \mathbf{A}_1^\top\right] +  \sigma_{\mathbf{A}_1}^2 \text{trace}\left(\mathbf{x} \mathbf{x}^\top\right) \mathbf{I} \\
                                                                                    & = \left(\frac{\epsilon^2 \sigma_{\mathbf{A}_1}^2 n}{3} + \sigma_{\mathbf{A}_1}^2 \text{trace}\left(\mathbf{x} \mathbf{x}^\top\right)\right) \mathbf{I}
    \end{align*}
    The last equality follows since:
    \begin{align*}
        \left(\mathbb{E}\left[\mathbf{A}_1 \mathbf{x} \mathbf{x}^\top \mathbf{A}_1^\top\right]\right)_{i,j} & = \mathbb{E}\left[\mathbf{a}_i^\top \mathbf{x} \mathbf{x}^\top \mathbf{a}_j\right] = \text{trace}\left(\mathbf{x}\mathbf{x}^\top \mathbb{E}\left[\mathbf{a}_j \mathbf{a}_i^\top\right]\right) \\
                                                                                                            & =
        \begin{cases}
            0                                                                          & \text{if} ~~ i \neq j \\
            \sigma_{\mathbf{A}_1}^2 \text{trace}\left(\mathbf{x}\mathbf{x}^\top\right) & \text{if} ~~ i = j
        \end{cases}
    \end{align*}
}

\begin{theorem}(ETB as Supersets in Expectation) Let Assumption \ref{key_assumption} hold. For a sufficiently large input dimension $n$,
    \begin{equation}
        \setlength{\abovedisplayskip}{-0.25pt}     \setlength{\abovedisplayshortskip}{-0.25pt}
        \begin{aligned}
            \mathbb{E}_{\mathbf{A}_1,\mathbf{a}_2}\left[\mathbf{L}_{\mathbf{ETB}}\right]  \leq \mathbb{E}_{\mathbf{A}_1,\mathbf{a}_2}\left[\mathbf{L}_{\text{\text{true}}}\right] & ~~~\text{ and } & \mathbb{E}_{\mathbf{A}_1,\mathbf{a}_2}\left[\mathbf{U}_{\text{\text{true}}}\right] \leq \mathbb{E}_{\mathbf{A}_1,\mathbf{a}_2}\left[\mathbf{U}_{\mathbf{ETB}}\right].
        \end{aligned}
    \end{equation}
\end{theorem}
\proof{
\begin{align*}
    \mathbf{L}_{\text{approx}}
     & \approx \mathbb{E}_{\mathbf{a}_2,\tilde{\mathbf{y}}} \left[\mathbf{a}_2^\top \text{max}\left(\tilde{\mathbf{y}},\mathbf{0}\right) + b_2\right]  - m \sqrt{\text{Var}_{\mathbf{a}_2,\tilde{\mathbf{y}}} \left[\mathbf{a}_2^\top \text{max}\left(\tilde{\mathbf{y}},\mathbf{0}\right) + b_2 \right]}                                                                     \\
     & \stackrel{\circledOne}{=} b_2 - m \Big(\mathbb{E}_{\mathbf{a}_2}\left[\text{Var}_{\tilde{\mathbf{y}}}\left(\mathbf{a}_2^\top \text{max}\left(\tilde{\mathbf{y}},\mathbf{0}\right) + b_2|\mathbf{a}_2\right)\right]                                                                                                                                                     \\
     & \quad\quad\quad\quad \quad + \text{Var}_{\mathbf{a}_2}\left(\mathbb{E}_{\tilde{\mathbf{y}}}\left[\mathbf{a}_2^\top \text{max}\left(\tilde{\mathbf{y}},\mathbf{0}\right)+b_2 | \mathbf{a}_2\right]\right)\Big)^{\frac{1}{2}}                                                                                                                                            \\
     & = b_2 - m \Bigg(\mathbb{E}_{\mathbf{a}_2}\left[(\mathbf{a}_2^\top\odot \mathbf{a}_2^\top) \left(\mathbb{E}_{\tilde{\mathbf{y}}} \left[\text{max}^2\left(\tilde{\mathbf{y}},\mathbf{0}\right) \right] - \left(\mathbb{E}_{\tilde{\mathbf{y}}}\left[\text{max}\left(\tilde{\mathbf{y}},\mathbf{0}\right)\right]\right)^2\right)\right]                                   \\
     & \quad\quad\quad\quad \quad  + \text{Var}_{\mathbf{a}_2}\left(\mathbb{E}_{\tilde{\mathbf{y}}}\left[\mathbf{a}_2^\top \text{max}\left(\tilde{\mathbf{y}},\mathbf{0}\right)+b_2 | \mathbf{a}_2\right]\right)\Bigg)^{\frac{1}{2}}                                                                                                                                          \\
     & = b_2 - m \Bigg(\left[\sum_{i=1}^k \sigma_{\mathbf{a}_2}^2 \left(\mathbb{E}_{\tilde{\mathbf{y}}} \left[\text{max}^2\left(\tilde{\mathbf{y}},\mathbf{0}\right) \right]\right)_i - \sum_{i=1}^k \sigma_{\mathbf{a}_2}^2 \left(\mathbb{E}_{\tilde{\mathbf{y}}}\left[\text{max}\left(\tilde{\mathbf{y}},\mathbf{0}\right)\right]\right)^2_i\right]                         \\& \quad\quad \quad\quad \quad  + \sum_{i=1}^k \sigma_{\mathbf{a}_2}^2 \left(\mathbb{E}_{\tilde{\mathbf{y}}}\left[ \text{max}\left(\tilde{\mathbf{y}},\mathbf{0}\right)\right]\right)^2_i\Bigg)^{\frac{1}{2}} \\
     & = b_2 - m \sigma_{\mathbf{a}_2} \sqrt{\sum_{i=1}^k  \left(\mathbb{E}_{\tilde{\mathbf{y}}} \left[\text{max}^2\left(\tilde{\mathbf{y}},\mathbf{0}\right) \right]\right)_i}                                                                                                                                                                                               \\
     & \stackrel{\circledTwo}{=} b_2 - m \sigma_{\mathbf{a}_2} \underbrace{\sqrt{\left(\mathbf{b}_1^{2i} + \sigma_{\tilde{\mathbf{y}}}^2\right) \odot \Phi\left(\mathbf{b}_1^i \oslash \sigma_{\tilde{\mathbf{y}}}\right) + \left(\mathbf{b}_1^i \odot \sigma_{\tilde{\mathbf{y}}} \odot \phi\left(\mathbf{b}_1^i \oslash \sigma_{\tilde{\mathbf{y}}}\right)\right)}}_{\Psi}.
\end{align*}
Note that $\circledOne$ follows by total expectation and total variance on the two terms, respectively. Lastly, $\circledTwo$ follows from the closed form expression derived in \citet{bibi2018analytic} where $\Phi$ and $\phi$ are the normal cumulative and probability Gaussian density functions, respectively. Note that $\tilde{\mathbf{y}} \sim \mathcal{N} \left(\mathbf{b}_1,\left(\frac{\epsilon^2 \sigma_{\mathbf{A}_1}^2 n}{3} +  \sigma_{\mathbf{A}_1}^2 \text{trace}\left(\mathbf{x} \mathbf{x}^\top\right)\right) \mathbf{I}\right)$ and that $\sigma_{\tilde{\mathbf{y}}}^2 = \left(\frac{\epsilon^2 \sigma_{\mathbf{A}_1}^2 n}{3} +  \sigma_{\mathbf{A}_1}^2 \text{trace}\left(\mathbf{x} \mathbf{x}^\top\right)\right)$.

\begin{align*}
    \mathbf{L}_{\text{approx}} & - \mathbb{E}_{\mathbf{A}_1,\mathbf{a}_2}\left[\mathbf{L}_{\mathbf{ETB}}\right]                                                                                                                                                                        \\
                               & \approx  b_2 - m \sigma_{\mathbf{a}_2} \Psi  - \mathbb{E}_{\mathbf{A}_1,\mathbf{a}_2}\left[\mathbf{a}_2^\top \mathbf{M}\left(\mathbf{A}_1\mathbf{x} + \mathbf{b}_1\right) + b_2 - \epsilon |\mathbf{a}_2^\top\mathbf{M}\mathbf{A}_1|\mathbf{1}\right] \\
                               & = \mathbb{E}_{\mathbf{A}_1,\mathbf{a}_2}\left[\epsilon |\mathbf{a}_2^\top\mathbf{M}\mathbf{A}_1|\mathbf{1}\right]  - m \sigma_{\mathbf{a}_2} \Psi                                                                                                     \\
                               & = \epsilon\mathbb{E}_{\mathbf{A}_1}\left[ \sum_{j=1}^n \mathbb{E}_{\mathbf{a}_2} \left[|\mathbf{a}_2^\top \mathbf{M} \mathbf{A}_1(:,j)| \big| \mathbf{A}_1\right]\right]   - m \sigma_{\mathbf{a}_2} \Psi                                             \\
                               & \stackrel{\circledOne}{=} \epsilon \sqrt{\frac{2}{\pi}}\mathbb{E}_{\mathbf{A}_1}\left[ \sum_{j=1}^n \sqrt{\text{Var}_{\mathbf{a}_2}\left(\mathbf{a}_2^\top \mathbf{M}\mathbf{A}_1(:,j))\right)}\right]  - m \sigma_{\mathbf{a}_2} \Psi                \\
                               & = \epsilon \sigma_{\mathbf{a}_2} \sqrt{\frac{2}{\pi}}\mathbb{E}_{\mathbf{A}_1}\left[ \sum_{j=1}^n \sqrt{\mathbf{A}_1(:,j)^\top \mathbf{M} \mathbf{A}_1(:,j)}\right]  - m \sigma_{\mathbf{a}_2} \Psi                                                   \\
                               & = \epsilon \sigma_{\mathbf{a}_2} \sqrt{\frac{2}{\pi}} \mathbb{E}_{\mathbf{A}_1}\left[ \sum_{j=1}^n \sqrt{\sum_{i=1}^k \mathbf{A}_1(i,j)^2 \mathbbm{1}\left\{\mathbf{u}_1^i \ge 0\right\}}\right]  - m \sigma_{\mathbf{a}_2} \Psi                      \\
                               & = \epsilon \sigma_{\mathbf{a}_2} \sqrt{\frac{2}{\pi}}\mathbb{E}_{|S|}\left[\mathbb{E}_{\mathbf{A}_1}\left[ \sum_{j=1}^n \sqrt{\sum_{i \in S}^k \mathbf{A}_1(i,j)^2 }\right] |\Bigg| |S|\right]  - m \sigma_{\mathbf{a}_2} \Psi                        \\
\end{align*}

\noindent Note that $\circledOne$ follows from the mean of a folded Gaussian. The last equality follows by taking the total expectation where $S$ is the set of indices where $\mathbf{u}_1^i \ge \mathbf{0}$ for all $i \in S$. Since $\mathbf{u}_1$ is random, then $|S|$ is also random. Therefore, one can reparameterize the sum and thus we have:
\begin{align*}
     & \epsilon \sigma_{\mathbf{a}_2} \sqrt{\frac{2}{\pi}}\mathbb{E}_{|S|}\left[\mathbb{E}_{\mathbf{A}_1}\left[ \sum_{j=1}^n \sqrt{\sum_{i \in S}^k \mathbf{A}_1(i,j)^2 }\right] |\Bigg| |S|\right] \\
     & =\frac{2\epsilon \sigma_{\mathbf{A}_1} \sigma_{\mathbf{a}_2} n}{\sqrt{\pi}}\mathbb{E}_{|S|}\left[\frac{\Gamma\left(\frac{|S| +1 }{2}\right)}{\Gamma\left(\frac{|S|}{2}\right)}\right]
    \ge \frac{2\epsilon \sigma_{\mathbf{A}_1} \sigma_{\mathbf{a}_2} n}{\sqrt{\pi}}\mathbb{E}_{|S|}\left[\frac{\sqrt{2} \pi}{\text{e}^{3}}\left(\frac{|S|}{|S|-2}\right)^{\frac{|S|}{2}}\sqrt{|S|-2}\right]
\end{align*}
The inequality follows by Stirling's formula where the right hand side only grows by order $n \mathbb{E}_{|S|}\sqrt{|S|}$ for large $|S|$. Then we have that:
\begin{align*}
     & \mathbf{L}_{\text{approx}} -  \mathbb{E}_{\mathbf{A}_1,\mathbf{a}_2}\left[\mathbf{L}_{\mathbf{ETB}}\right] \ge \underbrace{\frac{2 \sqrt{2\pi}\epsilon \sigma_{\mathbf{A}_1} \sigma_{\mathbf{a}_2} n}{\text{e}^{3}}\mathbb{E}_{|S|}\left[\left(\frac{|S|}{|S|-2}\right)^{\frac{|S|}{2}}\sqrt{|S|-2}\right] }_{\circledOne}                                                                                  \\
     & \quad \quad \quad \quad \quad - m \sigma_{\mathbf{a}_2} \underbrace{\Big(\sum_{i=1}^k  \left(\mathbf{b}_1^{2i} + \sigma_{\tilde{\mathbf{y}}}^2\right) \odot \Phi\left(\mathbf{b}_1^i \oslash \sigma_{\tilde{\mathbf{y}}}\right) + \left(\mathbf{b}_1^i \odot \sigma_{\tilde{\mathbf{y}}} \odot \phi\left(\mathbf{b}_1^i \oslash \sigma_{\tilde{\mathbf{y}}}\right)\right)\Big)^{\frac{1}{2}}}_{\circledTwo}
\end{align*}
Note that if $n$ grows sufficiently faster than $\sqrt{k}$, it grows faster than $\mathbb{E}_{|S|}\sqrt{|S|}$ since $\mathbb{E}_{|S|}[\sqrt{|S|}]\leq \sqrt{k}$ we have that
$\circledOne$ is $\mathcal{O}(n)$ while $\circledTwo$ is $\mathcal{O}(\sqrt{n})$. Thus, for sufficiently large input dimension $n$ we have that $\mathbf{L}_{\text{approx}} \ge \mathbb{E}_{\mathbf{A}_1,\mathbf{a}_2}[\mathbf{L}_{\mathbf{ETB}}]$ and since by construction $\mathbb{E}_{\mathbf{A}_1,\mathbf{a}_2}\left[\mathbf{L}_{\text{true}}\right]\ge \mathbf{L}_{\text{approx}}$ the proof is complete. Note that a symmetric argument can be applied to show that $\mathbb{E}_{\mathbf{A}_1,\mathbf{a}_2}\left[\mathbf{U}_{\mathbf{ETB}}\right] \ge \mathbb{E}_{\mathbf{A}_1,\mathbf{a}_2}\left[\mathbf{U}_{\text{true}}\right]$.
}

\begin{theorem}(ETB vs. IBP in Expectation)
    \label{supp:theorem_tightness}
    Consider an $\epsilon-\ell_\infty$ bounded uniform random variable input $\tilde{\mathbf{x}} \in \left[\mathbf{x}-\epsilon\mathbf{1}_n, \mathbf{x}+ \epsilon\mathbf{1}_n\right]$ to a block of layers in the form Affine-ReLU-Affine (parameterized by $\mathbf{A}_1, \mathbf{b}_1, \mathbf{a}_2$ and $ \mathbf{b}_2$ for the first and second affine layers respectively) and $\mathbf{a}_2^2\sim\mathcal{N}(\mathbf{0},\sigma_{\mathbf{a}_2}\mathbf{I})$. Under the assumption that $\frac{1}{\sqrt{2\pi}} \mathbf{x}_j \mathbf{1}_k^\top \mathbf{A}_1(:,j) + \frac{1}{2n} \mathbf{1}_k^\top \mathbf{b}_1 \ge \epsilon\left( \|\mathbf{A}_1(:,j)\|_2 - \frac{1}{\sqrt{2\pi}} \|\mathbf{A}_1(:,j)\|_1\right)$ $\forall j$, we have: $\mathbb{E}_{\mathbf{a}_2} \left[(\mathbf{U}_{\textbf{IBP}} - \mathbf{L}_{\textbf{IBP}}) - (\mathbf{U}_\mathbf{ETB} - \mathbf{L}_\mathbf{ETB})\right] \ge 0$.
\end{theorem}

\proof{
    Note that
    \begin{align*}
        \left[(\mathbf{U}_{\textbf{IBP}} - \mathbf{L}_{\textbf{IBP}}) - (\mathbf{U}_\mathbf{ETB} - \mathbf{L}_\mathbf{\textbf{M}})\right] & = \epsilon|\mathbf{a}_2^\top| |\mathbf{A}_1| \mathbf{1}_n + \frac{1}{2} |\mathbf{a}_2^\top| |\mathbf{u}_1| - \frac{1}{2} |\mathbf{a}_2^\top| |\mathbf{l}_1| \\
                                                                                                                                          & - 2 \epsilon |\Big|\mathbf{a}_2^\top \text{diag}\left(\mathbbm{1}\left\{\mathbf{u}_1 \ge \mathbf{0}\right\}\right) \mathbf{A}_1|\Big| \mathbf{1}_n
    \end{align*}
    Consider the coordinate splitting functions $S^{++}(.)$, $S^{+-}(.)$, $S^{--}(.)$ and $S^{-+}(.)$ such that for $\mathbf{x} \in \mathbb{R}^n$, $S^{++}(\mathbf{x}) = \mathbf{x} \odot {\mathbbm{1}\left\{\mathbf{u}_1^i \ge 0, \mathbf{l}^i_1 \ge 0\right\}}$ where $\mathbbm{1}\left\{\mathbf{u}_1^i \ge 0, \mathbf{l}^i_1 \ge 0\right\}$ is a vector of all zeros and 1 in the locations where both $\mathbf{u}_1^i,\mathbf{l}_1^i \ge 0$. However, since $\mathbf{u}_1 \ge \mathbf{l}_1$, then $S^{-+}(.) = \mathbf{0}$. Therefore it is clear that for any vector $\mathbf{x}$ and an interval $[\mathbf{l}_1,\mathbf{u}_1]$, we have that
    \begin{equation}
        \begin{aligned}
            \label{supp_index_split}
            \mathbf{x} = S^{++}\left(\mathbf{x}\right) + S^{+-}\left(\mathbf{x}\right) + S^{--}\left(\mathbf{x}\right),
        \end{aligned}
    \end{equation}
    since the sets $\{i;\mathbf{u}_1^i \ge 0, \mathbf{l}_1^i \ge 0\}$, $\{i; \mathbf{u}_1^i \ge 0, \mathbf{l}_1^i \le 0\}$ and $\{i; \mathbf{u}_1^i \leq 0, \mathbf{l}_1^i \leq 0\}$ are disjoints and their union $\{i=1,2,\dots,i=n\}$. We will denote the difference in the interval lengths as $W_{IBP} - W_{ETB}$ for ease of notation. Thus, we have:

    \begin{align*}
        W_{IBP} - W_{ETB} & = \epsilon S^{++}\left(|\mathbf{a}^\top_2|\right)|\mathbf{A}_1| \mathbf{1}_n + \epsilon S^{+-}\left(|\mathbf{a}^\top_2|\right)|\mathbf{A}_1| \mathbf{1}_n                                 \\&+ \epsilon S^{--}\left(|\mathbf{a}^\top_2|\right)|\mathbf{A}_1| \mathbf{1}_n + \frac{1}{2}S^{++}\left(|\mathbf{a}_2^\top|\right)|\mathbf{u}_1| \\
                          & + \frac{1}{2}S^{+-}\left(|\mathbf{a}_2^\top|\right)|\mathbf{u}_1| + \frac{1}{2}S^{--}\left(|\mathbf{a}_2^\top|\right)|\mathbf{u}_1|                                                       \\&- \frac{1}{2} S^{++}\left(|\mathbf{a}_2^\top|\right)|\mathbf{l}_1| -  \frac{1}{2} S^{+-}\left(|\mathbf{a}_2^\top|\right)|\mathbf{l}_1| -  \frac{1}{2} S^{--}\left(|\mathbf{a}_2^\top|\right)|\mathbf{l}_1|  \\&- 2 \epsilon \Bigg|\left(S^{++}\left(\mathbf{a}_2^\top\right) + S^{+-}\left(\mathbf{a}_2^\top\right) + S^{--}\left(\mathbf{a}_2^\top\right) \right)\text{diag}\left(\mathbbm{1}\left\{\mathbf{u}_1 \ge \mathbf{0} \right\}\mathbf{A}_1\right)\Bigg| \mathbf{1}_n \\
                          & = 2 \epsilon  S^{++}\left(|\mathbf{a}_2^\top|\right)|\mathbf{A}_1|\mathbf{1}_n + S^{+-}\left(|\mathbf{a}_2^\top|\right)\left(\mathbf{A}_1\mathbf{x} + \mathbf{b}_1\right)                 \\&+ \epsilon S^{+-}\left(|\mathbf{a}_2^\top|\right)\mathbf{A}_1 \mathbf{1}_n  \\
                          & - 2 \epsilon \Bigg|\left(S^{++}\left(\mathbf{a}_2^\top\right) + S^{+-}\left(\mathbf{a}_2^\top\right) \right)\mathbf{A}_1\Bigg| \mathbf{1}_n                                               \\
                          & = \underbrace{2 \epsilon  S^{++}\left(|\mathbf{a}_2^\top|\right)|\mathbf{A}_1|\mathbf{1}_n}_{\circledOne} + \underbrace{S^{+-}\left(|\mathbf{a}_2^\top|\right)\mathbf{u}_1}_{\circledTwo} \\&- \underbrace{2 \epsilon \Bigg|\left(S^{++}\left(\mathbf{a}_2^\top\right) + S^{+-}\left(\mathbf{a}_2^\top\right) \right)\mathbf{A}_1\Bigg| \mathbf{1}_n}_{\circledThree}.
    \end{align*}

    Note that in the first equality, we used the property of the coordinate splitting functions defined in Equation (\ref{supp_index_split}). In the second equality, we used the fact that $\mathbf{l}_1,\mathbf{u}_1 = \mathbf{A}_1\mathbf{x} + \mathbf{b}_1 \mp \epsilon |\mathbf{A}_1|\mathbf{1}_n$. The last term in the penultimate equality follows since $S^{++}$ and $S^{+-}$ corresponds to the indices that are selected by $\text{diag}\left(\mathbbm{1}\left\{\mathbf{u}_1 \ge \mathbf{0}\right\}\right)$.

    Now by taking the expectation over $\mathbf{a}_2$, we have for $\circledOne$:
    \begin{align*}
        2\epsilon \mathbb{E}\left[S^{++}\left(|\mathbf{a}_2^\top| |\mathbf{A}_1|\right)\right]\mathbf{1}_n & = 2 \epsilon \sum_{i=1}^k \mathbb{E}\left[|\mathbf{a}_2^i|\right] |\mathbf{A}_1(i,:)| \mathbbm{1}\left\{\mathbf{u}_1^i\ge\mathbf{0}, \mathbf{l}_1^i \ge 0\right\}\mathbf{1}_n \\
                                                                                                           & = 2 \epsilon \sigma_{\mathbf{a}_2} \sqrt{\frac{2}{\pi}} \sum_{i=1}^k |\mathbf{A}_1(i,:)| \mathbbm{1}\left\{\mathbf{l}_1^i \ge 0\right\}\mathbf{1}_n                           \\
                                                                                                           & = 2 \epsilon \sigma_{\mathbf{a}_2} \sqrt{\frac{2}{\pi}} \sum_{j=1}^n \sum_{i=1}^k |\mathbf{A}_1(i,j)|\mathbbm{1}\left\{\mathbf{l}_1^i \ge 0\right\}
    \end{align*}
    The second equality follows from the mean of the folded Gaussian and the fact that $\mathbf{u}_1 \ge \mathbf{l}_1$.

    For \circledTwo, we have:
    \begin{align*}
        \mathbb{E}\left[S^{+-}\left(|\mathbf{a}_2^\top|\right) \mathbf{u}_1\right] & = \sigma_{\mathbf{a}_2} \sqrt{\frac{2}{\pi}} \sum_{i=1}^k  \mathbf{u}_1^i \mathbbm{1}\left\{\mathbf{u}_1^i\ge0, \mathbf{l}_1^i \le 0\right\}
    \end{align*}
    Lastly, for \circledThree, we have:

    \begin{align*}
         & 2 \epsilon \mathbb{E} \left[ \Bigg|\left(S^{++}\left(\mathbf{a}_2^\top\right) + S^{+-}\left(\mathbf{a}_2^\top\right) \right)\mathbf{A}_1\Bigg| \right]\mathbf{1}_n = 2 \epsilon \mathbb{E} \Bigg| \left[\sum_{i=1}^k \mathbf{A}_1(i,:) \mathbf{a}_2^i \left(\mathbbm{1}\left\{\mathbf{u}_1^i \ge 0\right\} \right)\right]\Bigg| \mathbf{1}_n \\
    \end{align*}

    Using Holder's inequality, i.e. $\mathbb{E}[|x|] \leq \sqrt{\mathbb{E}[x^2]}$, per coordinate of the vector $ \left[\sum_{i=1}^k \mathbf{A}_1(i,:) \mathbf{a}_2^i \left(\mathbbm{1}\left\{\mathbf{u}_1^i \ge 0\right\} \right)\right]$ and by binomial expansion, we have at the $j^{\text{th}}$ coordinate

    \begin{align*}
         & 2 \epsilon \mathbb{E} \Bigg| \left[\sum_{i=1}^k \mathbf{A}_1(i,:) \mathbf{a}_2^i \left(\mathbbm{1}\left\{\mathbf{u}_1^i \ge 0\right\} \right)\right]\Bigg| \leq 2 \epsilon  \sqrt{\mathbb{E}\left[\sum_{i=1}^k\mathbf{A}_1(i,j)\mathbf{a}_2^i \mathbbm{1}\left\{\mathbf{u}_1^i \ge 0\right\}\right]^2}               \\
         & = 2 \epsilon \left(\sum_{i=1}^k \left(\mathbf{A}_1(i,j)\right)^2 \mathbb{E}\left[\Bigg(\mathbf{a}_2^i\right)^2\right] \mathbbm{1}\left\{\mathbf{u}_1^i \ge 0\right\}                                                                                                                                                 \\&\quad\quad\quad + 2\sum_{i=1} \sum_{z<i}\mathbf{A}_1(i,j) \mathbf{A}_1(z,j) \mathbb{E}\left[\mathbf{a}_2^i \mathbf{a}_2^z\right] \mathbbm{1}\left\{\mathbf{u}_1^i \ge 0\right\} \mathbbm{1}\left\{\mathbf{u}_1^z \ge 0\right\}\Bigg)^{\frac{1}{2}}\\
         & = 2 \epsilon  \sqrt{\sum_{i=1}^k \left(\mathbf{A}_1(i,j)\right)^2 \mathbb{E}\left[\left(\mathbf{a}_2^i\right)^2\right] \mathbbm{1}\left\{\mathbf{u}_1^i \ge 0\right\}}  = 2 \epsilon \sigma_{\mathbf{a}_2} \sqrt{\sum_{i=1}^k \left(\mathbf{A}(i,j)\right)^2  \mathbbm{1}\left\{\mathbf{u}_i \ge \mathbf{0}\right\}}
    \end{align*}

    The second equality follows by the independence of $\mathbf{a}_2^i$ and that they have zero mean. Therefore it follows from \circledThree that:
    \begin{align*}
         & 2 \epsilon \mathbb{E} \left[ \Bigg|\left(S^{++}\left(\mathbf{A}_2^\top\right) + S^{+-}\left(\mathbf{A}_2^\top\right) \right)\mathbf{A}_1\Bigg| \right]\mathbf{1}_n \leq 2 \epsilon \sigma_{\mathbf{a}_2}  \sum_{j=1}^n \sqrt{\sum_{i=1}^k \left(\mathbf{A}_1(i,j)\right)^2  \mathbbm{1}\left\{\mathbf{u}_i \ge \mathbf{0}\right\}} \\
    \end{align*}

    Lastly, putting things together, i.e. $\mathbb{E}\left[\circledOne + \circledTwo - \circledThree\right]$ we have that

    \begin{equation}
        \begin{aligned}
            \label{eq:sup_main_ineq}
            \mathbb{E}\left[W_{IBP} - W_{ETB}\right]
             & \ge  2 \epsilon \sigma_{\mathbf{a}_2} \sqrt{\frac{2}{\pi}} \sum_{j=1}^n \sum_{i=1}^k |\mathbf{A}_1(i,j)|\mathbbm{1}\left\{\mathbf{l}_1^i \ge 0\right\} \\&+ \sigma_{\mathbf{a}_2} \sqrt{\frac{2}{\pi}} \sum_{i=1}^k \mathbf{u}_1^i \mathbbm{1}\left\{\mathbf{u}_1^i \ge 0, \mathbf{l}_1^i \leq 0\right\}
            \\&- 2 \epsilon \sigma_{\mathbf{a}_2} \sum_{j=1}^n \sqrt{\sum_{i=1}^k \mathbf{A}_1(i,j)^2 \mathbbm{1}\left\{\mathbf{u}_1^i \ge 0 \right\}}.
        \end{aligned}
    \end{equation}

    Note that to show that the previous inequality is non-negative, it is sufficient to show that the previous inequality is non-negative for the non-intersecting sets $\{i:\mathbf{l}^i_1 \ge \mathbf{0}\}$ and $\{i:\mathbf{u}_1^i \ge \mathbf{0},\mathbf{l}_1^i \leq \mathbf{0}\}$. Thus the right hand side can be written as the sum of two sets.

    \textbf{For the set} $\{i:\mathbf{l}^i_1 \ge \mathbf{0}\}$, the RHS of inequality (\ref{eq:sup_main_ineq}) reduces to :

    \begin{equation}
        \begin{aligned}
            \label{sup_eq:first_set_inq}
            2\epsilon \sigma_{\mathbf{a}_2} \sum_{j=1}^n\left(\sqrt{\frac{2}{\pi}} \|\mathbf{A}_1(:,j)\|_1 - \|\mathbf{A}_1(:,j)\|_2 \right).
        \end{aligned}
    \end{equation}

    \textbf{For the set} $\{i:\mathbf{u}_1^i \ge \mathbf{0},\mathbf{l}_1^i \leq \mathbf{0}\}$ and using the definition of $\mathbf{u}_1$, the RHS of inequality (\ref{eq:sup_main_ineq}) reduces to

    \begin{align}
        \label{sup_eq:second_set_inq}
         & \sigma_{\mathbf{a}_2} \sqrt{\frac{2}{\pi}} \sum_{i=1}^k \left(\sum_{j=1}^n \mathbf{A}_1(i,j) \mathbf{x}_j + \mathbf{b}_i + \epsilon \sum_{j=1}^n |\mathbf{A}_1(i,j)|\right)- 2 \epsilon \sigma_{\mathbf{a}_2} \sum_{j=1}^n\|\mathbf{A}_1(:,j)\|_2 \nonumber                          \\
         & = \sigma_{\mathbf{a}_2} \sqrt{\frac{2}{\pi}} \sum_{j=1}^n \left(\mathbf{x}_j \mathbf{1}_k^\top \mathbf{A}_1(:,j) + \frac{1}{n}\mathbf{1}_k^\top \mathbf{b} + \epsilon \|\mathbf{A}_1(:,j)\|_1\right)- 2 \epsilon \sigma_{\mathbf{a}_2} \sum_{j=1}^n\|\mathbf{A}_1(:,j)\|_2 \nonumber \\
         & = \sum_{j=1}^n \left(\sigma_{\mathbf{a}_2} \sqrt{\frac{2}{\pi}} \Big(\mathbf{x}_j \mathbf{1}_k^\top \mathbf{A}_1(:,j) + \frac{1}{n}\mathbf{1}_k^\top \mathbf{b} + \epsilon \|\mathbf{A}_1(:,j)\|_1\Big) - 2 \epsilon \sigma_{\mathbf{a}_2} \|\mathbf{A}_1(:,j)\|_2\right)
    \end{align}

    Note that given the Assumption \ref{key_assumption} where $\frac{1}{\sqrt{2\pi}} \mathbf{x}_j \mathbf{1}_k^\top \mathbf{A}_1(:,j) + \frac{1}{2n} \mathbf{1}_k^\top \mathbf{b} \ge 0 \ge \epsilon\left( \|\mathbf{A}_1(:,j)\|_2 - \frac{1}{\sqrt{2\pi}} \|\mathbf{A}_1(:,j)\|_1\right)$ $\forall j$, then if both Equations (\ref{sup_eq:first_set_inq}) and (\ref{sup_eq:second_set_inq}), this completes the proof.
}

\begin{lemma}\label{supp_lemma_1}
    For $\mathbf{x} \in \mathbb{R}^k \sim \mathcal{N}\left(\mathbf{0},\mathbf{I}\right)$, where $k \ge 5$ we have that $\mathbb{E}\left[\frac{3}{\sqrt{2\pi}} \|\mathbf{x}\|_1 - 2\|\mathbf{x}\|_2 \right] \ge 0$.
\end{lemma}
\proof{
    Note that by the mean of a folded Gaussian, we gave that $  \mathbb{E}\left[\|\mathbf{x}\|_1\right] = \sum_i^k \mathbb{E}\left[|\mathbf{x}_i|\right] = k \sqrt{\frac{2}{\pi}}$. Moreover, note that
    \begin{align*}
        \mathbb{E}\left[\|\mathbf{x}\|_2\right] = \mathbb{E}\left[\sqrt{\sum_i^k \mathbf{x}_i^2}\right] & =  \mathbb{E} \left[\sqrt{y}\right] = \frac{1}{2^{\frac{k}{2}-1}\Gamma(\frac{k}{2})} \int_0^\infty x^{k} \exp\left(-\frac{x^2}{2}\right) dx \\&= \frac{2^\frac{k-1}{2} \Gamma\left(\frac{k+1}{2}\right)}{2^{\frac{k}{2}-1} \Gamma(\frac{k}{2})}
        = \sqrt{2}\frac{\Gamma\left(\frac{k+1}{2}\right)}{ \Gamma(\frac{k}{2})} \sim \sqrt{k}.
    \end{align*}
    Note that $y$ is Chi-Square random variable and that $f_{\sqrt{y}}(x) = 2x f_y(x^2) = \frac{x^{k-1}}{2^{\frac{k}{2}-1} \Gamma(\frac{k}{2})} \exp\left(-\frac{x^2}{2}\right)$ where the third inequality follows by integrating by parts recursively. Lastly, the last approximation follows by stirling's approximation for large $k$.
}

\begin{proposition}
    \label{supp:proposition_random_mat}
    For a random matrix $\mathbf{A}_1 \in \mathbb{R}^{k \times n}$ with i.i.d elements such $\mathbf{A}_1(i,j) \sim \mathcal{N}(0,1)$, then
    \begin{align*}
         & \mathbb{E}_{\mathbf{A}_1}\left(\|\mathbf{A}_1(:,j)\|_2 - \frac{1}{\sqrt{2\pi}} \|\mathbf{A}_1(:,j)\|_1\right) = \sqrt{2} \frac{\Gamma\left(\frac{k+1}{2}\right)}{\Gamma\left(\frac{k}{2}\right)} - k \sqrt{\frac{2}{\pi}} \approx \sqrt{k}\left(1 - \sqrt{\frac{2}{\pi}}\sqrt{k}\right).
    \end{align*}
\end{proposition}
\proof{
    The proof follows immediately from Lemma \ref{supp_lemma_1}.
}

\end{document}